\documentclass[10pt,twocolumn,letterpaper]{article}

\usepackage{iccv}
\usepackage{times}
\usepackage{epsfig}
\usepackage{graphicx}
\usepackage{subcaption}
\usepackage{amsmath}
\usepackage{amssymb}
\usepackage{amsmath}
\usepackage{amsthm}
\usepackage{amssymb}
\usepackage{mathtools}
\usepackage{xargs}[2008/03/08]
\usepackage{float,xcolor,booktabs,bbm,multirow,tabularx,algorithm2e}

\newcommand{\Lc}{\mathcal{L}}
\newcommand{\Dc}{\mathcal{D}}
\newcommand{\Sc}{\mathcal{S}}
\newcommand{\un}{\text{un}}
\newcommand{\unm}{\mathrm{un}}
\newcommand{\supp}{\text{sup}}

\newcommand{\bW}{\bar{W}}
\newcommand{\pos}{\mathrm{pos}}
\newcommand{\dat}{\mathrm{data}}
\newcommand{\E}{\mathbb{E}}
\newcommand{\Nc}{\mathcal{N}}

\newcolumntype{Y}{>{\centering\arraybackslash}X}
\newcommand{\iid}{\overset{\mathclap{\text{i.i.d.}}}{\sim}}
\newcommand{\norm}[1]{\Vert #1\Vert}

\theoremstyle{plain}
\newtheorem{thm}{\protect\theoremname}
\theoremstyle{plain}

\ifx\proof\undefined
\newenvironment{proof}[1][\protect\proofname]{\par
	\normalfont\topsep6\p@\@plus6\p@\relax
	\trivlist
	\itemindent\parindent
	\item[\hskip\labelsep\scshape #1]\ignorespaces
}{%
	\endtrivlist\@endpefalse
}
\providecommand{\proofname}{Proof}
\fi

\makeatother

\usepackage{babel}
\providecommand{\lemmaname}{Lemma}
\providecommand{\theoremname}{Theorem}

\usepackage[pagebackref=true,breaklinks=true,letterpaper=true,colorlinks,bookmarks=false]{hyperref}

\iccvfinalcopy 

\def\iccvPaperID{10380} 

\ificcvfinal\fi

\begin{document}

\title{Sharpness \& Shift-Aware Self-Supervised Learning}

\author{Ngoc N. Tran\\
VinAI Research
\and
Son Duong\\
VinAI Research
\and
Hoang Phan\\
VinAI Research
\and
Tung Pham\\
VinAI Research
\and
Dinh Phung\\
Monash University
\and
Trung Le\\
Monash University
}

\maketitle
\ificcvfinal\thispagestyle{empty}\fi

\begin{abstract}
Self-supervised learning aims to extract meaningful features from unlabeled data for further downstream tasks. In this paper, we consider classification as a downstream task in phase 2 and develop rigorous theories to realize the factors that implicitly influence the general loss of this classification task. Our theories signify that sharpness-aware feature extractors benefit the classification task in phase 2 and the existing data shift between the ideal (i.e., the ideal one used in theory development) and  practical (i.e., the practical one used in implementation) distributions to generate positive pairs also remarkably affects this classification task. Further harvesting these theoretical findings, we propose to minimize the sharpness of the feature extractor and a new Fourier-based data augmentation technique to relieve the data shift in the distributions generating positive pairs, reaching Sharpness $\&$ Shift-Aware Contrastive Learning (SSA-CLR). We conduct extensive experiments to verify our theoretical findings and demonstrate that sharpness $\&$ shift-aware contrastive learning can remarkably boost the performance as well as obtaining more robust extracted features compared with the baselines. The code for our experiments is publicly available at \url{https://anonymous.4open.science/r/ssa-clr}.
\end{abstract}

\section{Introduction}
Self-supervised learning (SSL) aims to extract useful representations from the input data without relying on human annotations, hence making model training more economic and efficient. Recent advances ~\cite{he2019moco, chen2020simple, grill2020byol, chen2020simsiam, caron2020unsupervised} in SSL show  excellent empirical evidences in various downstream tasks ranging from classification, object detection to instance segmentation with the comparable or even exceeding performance to supervised approaches.  

Contrastive learning (CLR) \cite{he2019moco, chen2020simple, misra2019contrastive, tian2019contrastive} is an essential technique in SSL in which positive and negative examples are created for each given anchor example. A feature extractor then learns to align representations of the anchors and their positive examples, while trying to contrast those of the anchors and their negative examples. SimCLR \cite{chen2020simple} is a pioneering work that proposed a simple yet but efficient technique to train a feature extractor with contrastive learning. In SimCLR, given an anchor, the positive examples are created using random data augmentations sampled from a pool $\mathcal{T}$ of data augmentations, while negative examples are simply sampled from data distribution. The InfoNCE loss is subsequently employed to train a feature extractor by aligning representations of positive pairs, while contrasting those of negative examples.

Inspired by surprising successes of SimCLR and other CLR techniques \cite{he2019moco, chen2020simple, misra2019contrastive, tian2019contrastive}, several works \cite{psaunshi19a, wang2020understanding, wang2022chaos} dedicated to study contrastive learning with the InfoNCE loss from a theoretical perspective. Specifically, \cite{psaunshi19a} established a connection between the general classification loss and the unsupervised loss in the context of binary classification using Rademacher complexity. Additionally, \cite{wang2020understanding} studied the distribution of the representations on the unit sphere and empirically and theoretically demonstrated that the representations generally tend to be uniformly distributed on the unit sphere while still maintaining the closeness of positive examples and their anchors. Recently, \cite{wang2022chaos} relieved the assumption made in \cite{wang2020understanding} to develop a connection between the general classification loss and the unsupervised loss in the context of multi-class classification. Furthermore, this work also indicated that the gap between two aforementioned losses would be further reduced if the pool of data augmentations $\mathcal{T}$ satisfies intra-class connectivity, which unfortunately is impossible to be realized without label information. Additionally, although providing more insightful understanding of contrastive learning with the InfoNCE loss, the aforementioned works need to make assumptions to some extent and none of them could yield any practical outcomes that help to improve the performance of contrastive learning.  

In this paper, sticking with the real-world setting contrastive learning without making any assumptions, we devise theories that glue the general supervised classification loss and the sharpness-aware unsupervised InfoNCE loss. Interestingly, our theories suggest that a sharpness-aware feature extractor with the InfoNCE loss can help to improve the generalization ability of the linear classifier in the classification downstream task. Moreover, through the developed theories, we observe that the data shift between the ideal and practical distributions to generate positive pairs also influences the performance of the classification downstream task. Harvesting this theoretical finding, we propose a Fourier-based data augmentation technique operated on top of data augmentations in SimCLR to exploit the inherent shift between the two aforementioned distributions. The underlying idea of our proposed Fourier-based data augmentation technique is to base on the feature extractor to find out the most likely same-label data example in the current batch for a given anchor. A Fourier transform \cite{Xu_2021_CVPR} is then applied to keep intact the foreground information of the anchor, while mixing up the background information of two relevant images. By this way, we expect that the new augmented image can capture better the space of the corresponding class, which would benefit the classification performance as suggested by our theories.

Finally, our contributions in this paper can be summarized as follows:
\begin{itemize}
    \item We develop rigorous theories for contrastive learning with the InfoNCE loss in the real-world setting without making any assumptions. Different from previous works \cite{psaunshi19a, wang2020understanding, wang2022chaos} , our theories establish a connection between the general supervised classification loss and the sharpness-aware unsupervised InfoNCE loss, hinting that minimizing the sharpness-aware unsupervised InfoNCE loss helps boosting classification performance. Still, our theories reveal the influence of the data shift between the ideal and practical distributions to generate positive pairs to the classification performance. 
    \item We harvest the sharpness-aware unsupervised InfoNCE loss and data shift to propose Sharpness \& Shift-Aware Contrastive Learning (SSA-CLR) to improve contrastive learning with the InfoNCE loss.
    \item We conduct experiments on real-world datasets to study the behaviors of our proposed components and compare our SSA-CLR with other baselines namely SimCLR \cite{chen2020simple}, Debiased CLR \cite{chuang2020debiased}, and Hard Negative Example Mining CLR \cite{robinson2020hard}, the works aim at improving contrastive learning with the InfoNCE loss. The experimental results show that our SSA-CLR significantly outperforms the baselines in the classification performance, while obtaining more robust features which are less vulnerable to adversarial attacks such as FGSM \cite{43405}.
\end{itemize}

\section{Related Works}
\subsection{Self-Supervised Learning}

Self-supervised learning is a learning paradigm that aims to learn meaningful representations of the input data without relying on human annotations. Recent advances in self-supervised learning for visual data~\cite{he2019moco, chen2020simple, grill2020byol, chen2020simsiam, caron2020unsupervised}  have demonstrated that these representations can be as effective as supervised representations in a range of downstream tasks, including classification, object detection, and instance segmentation.  The principle of these methods is finding representations that are invariant under different data augmentations. This is achieved by maximizing similarity of representations obtained from different augmented samples of an image. However, this process can potentially result in mode collapse~\cite{Jing2021UnderstandingDC}, where all images are mapped to the same representation through the network. To address this issue, several methods have been developed to learn more useful representations \cite{he2019moco, chen2020simple}.

Contrastive methods \cite{he2019moco, chen2020simple, misra2019contrastive, tian2019contrastive} create positive and negative pairs commonly by augmentations, and utilize InfoNCE loss which encourages representations of positive pairs to align and spreading which of negative pairs apart. However, these methods often require comparing each image with many others to achieve good results. In another recent line of work, BYOL~\cite{grill2020byol} and SimSiam~\cite{chen2020simsiam} only used positive pairs in the loss function, in which a special ``predictor''
network learns to output predictions aligning with the stop-gradient projections of another existing model. Meanwhile, SwAV~\cite{caron2020unsupervised} did not directly compare image features. Alternatively, it assigned augmentations from the same image to clusters, then simultaneously tried to enforce consistency between these clusters.


In typical contrastive learning methods, positive samples are obtained through  augmentation of the same images, while negative samples are selected from the remaining data. However, a potential issue arises when negative samples with the same label as the anchor sample are included in the selection process. To address this issue, a partial solution has been proposed by \cite{chuang2020debiased}, which introduced a distribution over the negative samples to correct potential biases and improve the quality of the representations. Following this work, \cite{robinson2020hard} incorporated similarity between the negative samples and the anchor into the sampling distribution to generate negative samples closer to the anchor, further enhancing the training process.

Driven by good empirical results of contrastive learning, several works have studied this learning paradigm from a theoretical perspective \cite{psaunshi19a, wang2020understanding, wang2022chaos}. Specifically, \cite{psaunshi19a} established the connection between the supervised and unsupervised losses using Rademacher complexity in the context of binary classification. Additionally, \cite{wang2020understanding} studied the distribution of latent representations over the unit sphere, and demonstrated that these latent representations tend to be uniformly distributed and encourage alignment between positive examples and their anchors. More recently, \cite{wang2022chaos} replaced the conditional independence assumption in \cite{psaunshi19a} by a milder condition, devised a connection between supervised and unsupervised losses, and rethought a new augmentation overlap theory for contrastive learning. 

Compared to existing works, ours is the first work that establishes the connection between the supervised classification loss and the sharpness-aware unsupervised InfoNCE loss without making any assumptions. Moreover, our theories reveal a distribution shift between the ideal and practical distributions to generate positive pairs which can be mitigated for improving supervised performances. More importantly, different from the aforementioned theoretical works, we harvest the theories to propose sharpness-aware contrastive learning and a novel Fourier-based data augmentation technique to significantly boost the classification performance and the robustness of the extracted features.

\subsection{Flat Minima}

Recent research in deep learning has focused on the importance of flat minima in improving the generalization ability of neural networks. Several seminal studies have demonstrated that wider local minima can lead to more robust models that are less likely to overfit or perform poorly on unseen data \cite{DBLP:conf/iclr/JiangNMKB20, DBLP:conf/nips/PetzkaKASB21, DBLP:conf/uai/DziugaiteR17}. To this end, various methods have been proposed to seek flat minima ~\cite{DBLP:conf/iclr/PereyraTCKH17, Chaudhari2017EntropySGDBG, DBLP:conf/iclr/KeskarMNST17, DBLP:conf/uai/IzmailovPGVW18, foret2021sharpnessaware}, including techniques that manipulate training factors such as batch-size, learning rate, dropout, and the covariance of gradient, as well as regularization techniques such as low entropy penalty \cite{DBLP:conf/iclr/PereyraTCKH17} and distillation losses ~\cite{DBLP:conf/iclr/KeskarMNST17,Jastrzebski2017ThreeFI, wei2020implicit}. In another line of work, \cite{DBLP:conf/uai/IzmailovPGVW18} demonstrated  that averaging model weights during training can yield flatter minima, which can lead to better generalization.

Motivated by the strong connection between sharpness of a minima and generalization
ability,  Sharpness-Aware Minimization (SAM) \cite{foret2021sharpnessaware} has emerged as a prominent approach for explicitly seeking flat regions in the loss landscape. SAM has been shown to be effective and scalable in a wide range of tasks and domains, including meta-learning~\cite{abbas2022sharp}, federated learning~\cite{qu2022generalized}, vision models~\cite{chen2021vision}, language models~\cite{bahri-etal-2022-sharpness}, domain generalization~\cite{cha2021swad}, and multi-task learning \cite{phan2022improving}. Other works have attempted to further enhance the effectiveness of SAM by exploiting its geometry \cite{kwon2021asam}, minimizing surrogate gap \cite{zhuang2022surrogate}, and speeding up training time \cite{du2022sharpness, liu2022towards}.

\section{Problem Formulation and Notions}
In this section, we present the problem formulation of self-supervised learning and the notions used in our following theory development.

We consider an $M$-class classification problem with the label set
$\mathcal{Y}=\left\{ 1,2,...,M\right\} $. Given a class $c\in\mathcal{Y}$,
the class-condition distribution for this class has the density function
$p_{c}\left(x\right)=p\left(x\mid y=c\right)$ where $x\in\mathbb{R}^{d}$
specifies a data example. Therefore, the entire data distribution
has the form
\[
p_{\text{data}}\left(x\right)=\sum_{c=1}^{M}\pi_{c}p\left(x\mid y=c\right)=\sum_{c=1}^{M}\pi_{c}p_{c}\left(x\right),
\]
where $\pi_{c}=\mathbb{P}\left(y=c\right),c\in\mathcal{Y}$ is a class
probability. 

The ideal distribution of positive pairs over $\mathbb{R}^{d}\times\mathbb{R}^{d}$
is formulated as 
\[
p_{\text{pos}}\left(x,x^{+}\right)=\sum_{c=1}^{M}\pi_{c}p_{c}\left(x\right)p_{c}\left(x^{+}\right).
\]
It is worth noting that with the above equality, $p_{\text{pos}}\left(x,x^{+}\right)$
is relevant to the probability that $x,x^{+}\sim p_{\text{data}}$
have the same label. Particularly, to form a positive pair $\left(x,x^{+}\right)$,
we first sample a class $c\sim\text{Cat}\left(\pi\right)$ from the
categorical distribution with $\pi=\left[\pi_{c}\right]_{c=1}^{M}$,
and then sample $x,x^{+}\sim p_{c}$. We now depart from $p_{\text{pos}}\left(x,x^{+}\right)$, the ideal distribution
of positive pairs to the practical distribution. Subsequently, we extend our theory to target a practical distribution of positive pairs $\tilde{p}_{\text{pos}}\left(x,x^{+}\right)$
whose samples are from random augmentations as in SimCLR \cite{chen2020simple}. 

The general unsupervised InfoNCE loss over the entire data and positive
pair distributions is denoted as 
{\small{}
\begin{align*}
 & \mathcal{L}_{\mathcal{D}_{\text{un}}}^{\text{un}}\left(\theta,p_{\text{pos}}\right)=\mathbb{E}_{\left(x,x^{+}\right)\sim p_{\text{pos}},x_{1:K}^{-}\iid p_{\text{data}}}\Biggl[\\
 & -\log\frac{\exp\left\{ \frac{f_{\theta}\left(x\right)\cdot f_{\theta}\left(x^{+}\right)}{\tau}\right\} }{\exp\left\{ \frac{f_{\theta}\left(x\right)\cdot f_{\theta}\left(x^{+}\right)}{\tau}\right\} +\frac{\beta}{K}\sum_{k=1}^{K}\exp\left\{ \frac{f_{\theta}\left(x\right)\cdot f_{\theta}\left(x_{k}^{-}\right)}{\tau}\right\} }\Biggr],
\end{align*}
}
where $f_{\theta}$ with $\theta\in\Theta$ is a feature extractor, the operation $f_{\theta}(x) \cdot f_{\theta}(\widetilde{x})$ means the inner product, $\tau>0$
is a temperature variable, $K$ is the number of negative examples
used, and $\mathcal{D}_{\text{un}}$ is the distribution over $z=\left[x,x^{+},\left[x_{k}^{-}\right]_{k=1}^{K}\right]$
with $\left(x,x^{+}\right)\sim p_{\text{pos}},x_{1:K}^{-} \sim p_{\text{data}}$.
Note that $\beta \geq 0$ is a parameter and setting $\beta=K$ recovers
the original formula of contrastive learning.

It is our ultimate goal to minimize the general unsupervised InfoNCE loss.
However, in reality, we work with a specific training set $\mathcal{S}=\Big\{ z_{i}=\Big[x_{i},x_{i}^{+},\big[x_{ik}^{-}\big]_{k=1}^{K}\Big]\Big\} _{i=1}^{N}$
where $z_{1:N}\sim\mathcal{D}_{\text{un}}$. The empirical unsupervised InfoNCE
loss over $\mathcal{S}$ is defined as
{\small{}
\begin{align*}
 & \mathcal{L}_{\mathcal{S}}^{\text{un}}\left(\theta,p_{\text{pos}}\right)=-\frac{1}{N}\times\\
 & \sum_{i=1}^{N}\log\frac{\exp\left\{ \frac{f_{\theta}\left(x_{i}\right)\cdot f_{\theta}\left(x_{i}^{+}\right)}{\tau}\right\} }{\exp\left\{ \frac{f_{\theta}\left(x_{i}\right)\cdot f_{\theta}\left(x_{i}^{+}\right)}{\tau}\right\} +\frac{\beta}{K}\sum_{k=1}^{K}\exp\left\{ \frac{f_{\theta}\left(x_{i}\right)\cdot f_{\theta}\left(x_{ik}^{-}\right)}{\tau}\right\} }.
\end{align*}
}

SSL aims to minimize the empirical unsupervised InfoNCE loss
over a specific training set $\mathcal{S}$ to learn an optimal feature extractor $f_{\theta^{*}}$ which will be used in
the second phase, wherein we train a linear classifier on top of the
features extracted by $f_{\theta^{*}}$. Given a feature extractor
$f_{\theta}$ and a weight matrix $W$ parameterized a linear classifier,
we define the general loss induced by this couple as
\begin{align*}
\mathcal{L}_{\mathcal{D}_{\text{sup}}}^{\text{sup}}\left(\theta,W\right)=\mathbb{E}_{\left(x,y\right)\sim\mathcal{D}_{\text{sup}}}\left[\tau_\text{CE}\left(Wf_{\theta}\left(x\right),y\right)\right],
\end{align*}
where $\mathcal{D}_{\text{sup}}\left(x,y\right)=\pi_{y}p_{y}\left(x\right)$
is the data-label distribution and $\tau_\text{CE}\left(\cdot,\cdot\right)$
is the $\tau$-temperature cross-entropy loss (i.e., softmax with temperature $\tau$ applied to logits before computing the cross-entropy loss). 

Given the fact that we aim to train the optimal linear classifier
in the phase 2, we define the optimal general loss over all weight
matrices $W$ as
\begin{align*}
\mathcal{L}_{\mathcal{D}_{\text{sup}}}^{\text{sup}}\left(\theta\right)=\min_{W}\mathcal{L}_{\mathcal{D}_{\text{sup}}}^{\text{sup}}\left(\theta,W\right).
\end{align*}

\section{Sharpness \& Shift-Aware Self-Supervised Learning}
In what follows, we present our theory development for SSL. We first establish the relevant theories for the ideal distribution $p_{\mathrm{pos}}(x, x^+)$ to generate positive pairs. Particularly, we specify a connection between the general supervised loss $\mathcal{L}_{\mathcal{D}_{\mathrm{sup}}}^{\mathrm{sup}}\left(\theta\right)$  in the second phase and the general unsupervised InfoNCE loss $\mathcal{L}_{\mathcal{D}_{\mathrm{un}}}^{\mathrm{un}}\left(\theta,p_{\mathrm{pos}}\right)$ in the first phase. From this connection, we devise a new theory to connect the general unsupervised loss $\mathcal{L}_{\mathcal{D}_{\mathrm{un}}}^{\mathrm{un}}\left(\theta,p_{\mathrm{pos}}\right)$ and the empirical sharpness-aware unsupervised InfoNCE loss  
$\max_{\theta':\Vert\theta^{\prime}-\theta \Vert<\rho}\mathcal{L}_{\mathcal{S}}^{\mathrm{un}}\left(\theta,\tilde{p}_{\mathrm{pos}}\right)$.

The above theories are developed for the ideal distribution $p_{\mathrm{pos}}(x, x^+)$ to generate positive pairs. By noticing that in practice, due to the lack of labels, we use the practical distribution $\tilde{p}_{\mathrm{pos}}(x, x^+)$ as an approximation of the ideal one to generate positive pairs based on random augmentations, we further extend our theories for the practical case when the practical distribution $\tilde{p}_{\mathrm{pos}}(x, x^+)$ is employed. Interestingly, our theory development emerges a term, standing for the shift between the ideal and practical distributions to generate positive pairs. Furthermore, hinted by the emerging positive distribution shift term, we propose a new Fourier-based random data augmentation technique to improve the performance of SSL approaches. 

To summarize, apart from the rigorous theory development to obtain insightful understanding of the factors that influence the performance of SSL, we harvest our theories to propose (i) sharpness-aware SSL and (ii) a new Fourier-based random data augmentation technique to further improve the performance of SSL, which can be empirically demonstrated via our experiments.

\subsection{Theory Development for the Ideal Distribution} \label{sec:ideal_case}
We now develop theories for the case of using the ideal distribution $p_{\mathrm{pos}}(x, x^+)$ to generate the positive pairs. Particularly, in Theorem \ref{thm:sup_un}, we indicate an upper-bound for the general supervised loss $\mathcal{L}_{\mathcal{D}_{\mathrm{sup}}}^{\mathrm{sup}}\left(\theta\right)$ which is relevant to the general unsupervised loss $\mathcal{L}_{\mathcal{D}_{\mathrm{un}}}^{\mathrm{un}}\left(\theta,p_{\mathrm{pos}}\right)$.

\begin{thm}
\label{thm:sup_un}The following inequality holds
\begin{align}
\mathcal{L}_{\mathcal{D}_{\mathrm{sup}}}^{\mathrm{sup}}\left(\theta\right) & \leq\mathcal{L}_{\mathcal{D}_{\mathrm{un}}}^{\mathrm{un}}\left(\theta,p_{\mathrm{pos}}\right)-O\left(\frac{1}{\sqrt{K}}\right)\nonumber \\
 & -\log\beta -O\Big(\frac{1}{\beta}\Big).\label{eq:Lsup_bound_Lun}
\end{align}
\end{thm}

Inequality (\ref{eq:Lsup_bound_Lun}) points out that to achieve
\[
\min_{\theta}\mathcal{L}_{\mathcal{D}_{\mathrm{sup}}}^{\mathrm{sup}}\left(\theta\right),
\]
we can alternatively minimize its upper-bound which is relevant to
$\mathcal{L}_{\mathcal{D}_{\mathrm{un}}}^{\mathrm{un}}\left(\theta\right)$. Unfortunately,
minimizing $\mathcal{L}_{\mathcal{D}_{\mathrm{un}}}^{\mathrm{un}}\left(\theta\right)$
directly is intractable due to the unknown general distribution $\mathcal{D}_{\mathrm{un}}$.
The following theorem resolves this issue and also signifies the concept
of sharpness for the feature extractor $f_{\theta}$.

\begin{thm}
\label{thm:sharpness}Under mild conditions, with the probability at least $1-\delta$ over
the random choice of $\mathcal{S}\sim\mathcal{D}_{\mathrm{un}}^{N}$, we have
the following inequality 

\begin{align*}
&\mathcal{L}_{\mathcal{D}_{\mathrm{un}}}^{\mathrm{un}}\left(\theta,p_{\mathrm{pos}}\right)  \leq\max_{\theta':\Vert\theta'-\theta\Vert<\rho}\mathcal{L}_{\mathcal{S}}^{\mathrm{un}}\left(\theta,p_{\mathrm{pos}}\right)+\\
&\quad \frac{1}{\sqrt{N}}\Biggl[  \frac{T}{2}\log\left(1+\frac{\norm{\theta}^{2}}{T\sigma^{2}}\right)+\log\frac{1}{\delta}+\frac{L^{2}}{8}+2L \\
&\quad + O\Big(\log(N+T)\Big)\Biggr],
\end{align*}
where $L=\frac{2}{\tau}+\log(1+\beta)$, $T$ is the number of parameters
in $\theta$, and $\sigma=\frac{\rho}{\sqrt{T}+\sqrt{\log\left(N\right)}}$.
\end{thm}

We note that the proof in \cite{foret2021sharpnessaware} invoked the McAllester PAC-Bayesian generalization bound \cite{mcallester1999pac}, hence only applicable to the 0-1 loss in the binary classification setting. Ours is the first work that proposes and devises sharpness-aware theory for feature extractor in the context of SSL. Additionally, the proof of our theory employs the PAC-Bayesian generalization bound \cite{JMLR:v17:15-290} to deal with the more general InfoNCE loss.

By leveraging Theorems \ref{thm:sup_un} and \ref{thm:sharpness},
we reach the following theorem.

\begin{thm}
\label{thm:main}Under mild conditions, with the probability at least $1-\delta$ over the
random choice of $\mathcal{S}\sim\mathcal{D}_\mathrm{un}^{N}$, we have the
following inequality 
\begin{align*}
&\mathcal{L}_{\mathcal{D}_{\mathrm{sup}}}^{\mathrm{sup}}\left(\theta\right)  \leq\max_{\theta':\Vert\theta'-\theta\Vert<\rho}\mathcal{L}_{\mathcal{S}}^\mathrm{un}\left(\theta',p_{\mathrm{pos}}\right)
  -O\left(\frac{1}{\sqrt{K}}\right)-\\ &\quad \log\beta-
  O\Big(\frac{1 }{\beta}\Big)
+\frac{1}{\sqrt{N}} \Bigg[\frac{T}{2}\log\left(1+\frac{\norm{\theta}^{2}}{T\sigma^{2}}\right)+\\
&\quad \log\frac{1}{\delta}+\frac{L^{2}}{8}+2L + O\Big(\log(N+T) \Big)\Bigg]
\end{align*}
where $L=\frac{2}{\tau} + \log(1+\beta)$, $T$ is the number of parameters
in $\theta$, and $\sigma=\frac{\rho}{\sqrt{T}+\sqrt{\log\left(N\right)}}$.
\end{thm}

Theorem \ref{thm:main} benefits us in two folds. First, it explains why in training SSL approaches, the extremely large batch size $K$ is necessary to reduce the gap between the general supervised loss and the unsupervised InfoNCE general loss for boosting the classification performance in the second phase. Second, it sheds lights for us to develop our sharpness-aware SSL for implicitly lowering the general loss and hence improving the classification performance in the second phase.

\subsection{Theory Development for the Practical Distribution} \label{sec:practical_case}
In Section \ref{sec:ideal_case}, we develop the theories for the ideal case when using the ideal distribution $p_\mathrm{pos}(x,x^+)$ to generate positive pairs. However, in practice, we employ a practical distribution $\tilde{p}_\mathrm{pos}(x,x^+)$ to approximate the ideal one. It is appealing to extend our theory development for this practical setting, further leading us to the awareness of the shift between two relevant distributions and the proposal of a new Fourier-based random data augmentation technique to reduce this gap.  

We first describe practical $\tilde{p}_{\mathrm{pos}}\left(x,x^{+}\right)$
based on random augmentations. Given a distribution over data augmentations
$\mathcal{T}$, we sample a specific data augmentation $t\sim\mathcal{T}$
to compute $x^{+}=t\left(x\right)$ and form the positive pair $\left(x,x^{+}\right)$.
The set of such random pairs form the practical distribution $\tilde{p}_{\mathrm{pos}}\left(x,x^{+}\right)$ over
positive pairs. We develop
the bound between $\mathcal{L}_{\mathcal{D}_\mathrm{sup}}^\mathrm{sup}\left(\theta\right)$
and $\mathcal{L}_{\mathcal{D}_\mathrm{un}}^\mathrm{un}\left(\theta,\tilde{p}_{\mathrm{pos}}\right)$
in the following theorem.

\begin{thm}
\label{thm:practical_sup_un}The following inequality holds

\begin{align} \label{shift_1}
 & \mathcal{L}_{\mathcal{D}_\mathrm{sup}}^\mathrm{sup}\left(\theta\right)\leq\mathcal{L}_{\mathcal{D}_\mathrm{un}}^\mathrm{un}\left(\theta,\tilde{p}_{\mathrm{pos}}\right)-O\left(\frac{1}{\sqrt{K}}\right) \nonumber\\
 & \qquad +\mathcal{L}_\mathrm{shift}\left(\tilde{p}_{\mathrm{pos}},p_{\mathrm{pos}}\right)-\log\beta-O\Big(\frac{1 }{\beta}\Big), 
\end{align} 
where $\mathcal{L}_\mathrm{shift}\left(\tilde{p}_{\mathrm{pos}},p_{\mathrm{pos}}\right)$
is defined as
{\smaller{}
\[
\tau\sum_{c=1}^{M}\pi_{c}\mathbb{E}_{x\sim p_{c}}\left[\Vert\mathbb{E}_{x^{+}\sim p_{c}}\left[f_{\theta}\left(x^{+}\right)\right]-\mathbb{E}_{t\sim\mathcal{T},x^{+}=t\left(x\right)}\left[f_{\theta}\left(x^{+}\right)\right]\Vert^{\frac{1}{2}}\right]. 
\]
}
\end{thm}

In the upper-bound (\ref{shift_1}), there appears the data shift term $\mathcal{L}_\mathrm{shift}\left(\tilde{p}_{\mathrm{pos}},p_{\mathrm{pos}}\right)$ between the ideal and practical distributions to generate positive pairs. Evidently, this term would be minimized when given a data example $x$ in the class $c$, we can strengthen the random data augmentation $t \sim \mathcal{T}$ so that the positive examples $x^+ = t(x)$  are more diverse to capture better the class-condition distribution $p_c$. However, it is a challenging task because we do not possess any label information of any class $c$ to characterize the space of this class. Here we note that the definition of $\mathcal{L}_{\mathcal{D}_\mathrm{un}}^\mathrm{un}\left(\theta,\tilde{p}_{\mathrm{pos}}\right)$ is the same as that of $\mathcal{L}_{\mathcal{D}_\mathrm{un}}^\mathrm{un}\left(\theta,p_{\mathrm{pos}}\right)$ except that positive pairs $(x,x^+)$ are sampled from the practical distribution $\tilde{p}_{\mathrm{pos}}$. 

We finally develop the following theorem to glue $\mathcal{L}_{\mathcal{D}_\mathrm{sup}}^\mathrm{sup}\left(\theta\right)$
and the sharpness-aware unsupervised InfoNCE loss over the practical
distribution $\tilde{p}_{\mathrm{pos}}\left(x,x^{+}\right)$ in the following theorem.

\begin{thm}
\label{thm:practical_sharpness}Under mild conditions, with the probability at least $1-\delta$
over the random choices $\mathcal{S}\sim\mathcal{D}_\mathrm{un}^{N}$, we
have the following inequality 
{\small{}
\begin{align*}
 & \mathcal{L}_{\mathcal{D}_\mathrm{sup}}^\mathrm{sup}\left(\theta\right)\leq\max_{\theta':\norm{\theta'-\theta}<\rho}\mathcal{L}_{\mathcal{S}}^\mathrm{un}\left(\theta',\tilde{p}_{\mathrm{pos}}\right)-O\left(\frac{1}{\sqrt{K}}\right)\\
 & +\frac{1}{\sqrt{N}}\Bigg[\frac{T}{2}\log\left(1+\frac{\norm{\theta}^{2}}{T\sigma^{2}}\right)+\log\frac{1}{\delta}+\frac{L^{2}}{8}+2L + \\
 &O\Big(\log(N+T) \Big)\Bigg]
+\mathcal{L}_\mathrm{shift}\left(\tilde{p}_{\mathrm{pos}},p_{\mathrm{pos}}\right)-\log\beta-O\Big(\frac{1 }{\beta}\Big)
\end{align*}
}
where $L=\frac{2}{\tau}+\log(1+\beta)$, $T$ is the number of parameters
in $\theta$, and $\sigma=\frac{\rho}{\sqrt{T}+\sqrt{\log\left(N\right)}}$.
\end{thm}

\subsection{Exploiting Theories for a Practical Method}
We now harvest our developed theories to reach a practical method improving SSL approaches based on the InfoNCE loss.  Based on Theorem \ref{thm:practical_sharpness}, we use one gradient ascent step to find $\theta^a$  and update the current model $\theta$ as
\begin{align*}
\theta^{a} & =\theta+\rho\frac{\nabla_{\theta}\mathcal{L}_{B}^\mathrm{un}\left(\theta,\tilde{p}_{\mathrm{pos}}\right)}{\Vert\nabla_{\theta}\mathcal{L}_{B}^\mathrm{un}\left(\theta,\tilde{p}_{\mathrm{pos}}\right)\Vert},\\
\theta & =\theta-\eta\nabla_{\theta}\mathcal{L}_{B}^{un}\left(\theta^{a},\tilde{p}_{\mathrm{pos}}\right),
\end{align*}
where $B = \{x_1,\dots, x_b\}$ is the current batch, $\rho >0$ is the perturbation radius, and $\eta >0$ is the learning rate.

Our proposed random data augmentation technique relies on the Fourier transformation. Given a single-channel image $x$, we compute its Fourier transform $\mathcal{F}(x)$ as  
{\small{}
\[
\mathcal{F}\left(x\right)\left(u,v\right)=\sum_{h=0}^{H-1}\sum_{w=0}^{W-1}x\left(h,w\right)\exp\left( -2\pi i\left(\frac{h}{H}u+\frac{w}{W}v\right)\right), 
\]}
where $i$ is the imaginary unit.

We denote $\mathcal{F}^{-1}(x)$ as the inverse Fourier transform. Note that both the Fourier transformation and its inverse
can be calculated with the FFT algorithm \cite{walker2017fast} efficiently. The amplitude and phase components are represented as 
\begin{align*}
\mathcal{A}\left(x\right)\left(u,v\right) & =\left[R^{2}\left(x\right)\left(u,v\right)+I^{2}\left(x\right)\left(u,v\right)\right]^{\frac{1}{2}}\\
\mathcal{P}\left(x\right)\left(u,v\right) & =\mathrm{arctan}\left[\frac{I\left(x\right)\left(u,v\right)}{R\left(x\right)\left(u,v\right)}\right],
\end{align*}
where $R(x)$ and $I(x)$ are the real and imaginary parts of $\mathcal{F}(x)$ respectively. 

It is well-known that the phase contains the foreground information, while the amplitude contains the background information \cite{Xu_2021_CVPR}. We now present how to apply our proposed Fourier-based data augmentation technique to SSL. For each data example $x_i \in B$, we sample two random augmentations $t,t' \sim \mathcal{T}$ similar to SimCLR \cite{chen2020simple} to form its positive examples $\tilde{x}_{2i-1} = t(x_i)$ and $\tilde{x}_{2i} = t'(x_i)$. Moreover, for each $\tilde{x}_k,  k=1,\dots,2b$, we find the most-similar example in the batch of positive examples as
\[
\tilde{x}_{k}^{c}=\mathrm{argmax}{}_{l\neq k}f_{\theta}\left(\tilde{x}_{k}\right).f_{\theta}\left(\tilde{x}_{l}\right).
\]
By doing so, we hope that for a quite-well trained feature extractor $f_\theta$, $\tilde{x}_{k}^{c}$ has more likely the same label as $\tilde{x}_{k}$. Inspired by \cite{Xu_2021_CVPR}, we next apply the Fourier transform to $\tilde{x}_{k}$, $\tilde{x}_{k}^{c}$ and then apply linear interpolation to their amplitudes, while keeping the phases intact.  
\[
\hat{\mathcal{A}}\left(\tilde{x}_{k}\right)=(1-\beta)\mathcal{A}\left(\tilde{x}_{k}\right)+\beta\mathcal{A}\left(\tilde{x}_{k}^{c}\right),
\]
where the coefficient $\beta \sim \mathrm{Uniform}(0, \alpha)$. 

Finally, we replace the positive example $\tilde{x}_{k}$ by an another positive example $\hat{x}_{k}$ computed as $\hat{x}_{k}=\mathcal{F}^{-1}\left(\mathcal{F}\left(\hat{x}_{k}\right)\right)$ where we have defined
\[
\mathcal{F}\left(\hat{x}_{k}\right)\left(u,v\right)=\hat{\mathcal{A}}\left(\tilde{x}_{k}\right)\exp\left\{i\mathcal{P}\left(\tilde{x}_{k}\right)\left(u,v\right)\right\}. 
\]

Because we keep intact the phase of $\tilde{x}_k$, the foreground information of $\hat{x}_k$ is similar to that of $\tilde{x}_k$, while its background information is interfered by $\tilde{x}_{k}^{c}$, expecting to have the same label as $\tilde{x}_k$. As a result, $\hat{x}_k$ is more diverge than $\tilde{x}_k$ in capturing and characterizing other similar-labeled examples in the mini-batch. Finally, we make use of the Fourier-based positive examples $\hat{x}_1,\dots,\hat{x}_{2b}$ in the InfoNCE loss. 

The resulting augmented images are illustrated in \autoref{fig:fft}. In the first two rows, we mix up an image with an another image in the same class. Resultantly, for the mixed images, the foreground information from the original images is kept nearly intact while the background information is interfered with that of the second image in the same class. With this construction, the mixed images can characterize better the space of their classes. In the third row, we accidentally mix up an original image with the one in a different class. However, the mixed image still maintains the crucial information of the original one.

\begin{figure}[t]
    \centering
    \includegraphics[width=\linewidth]{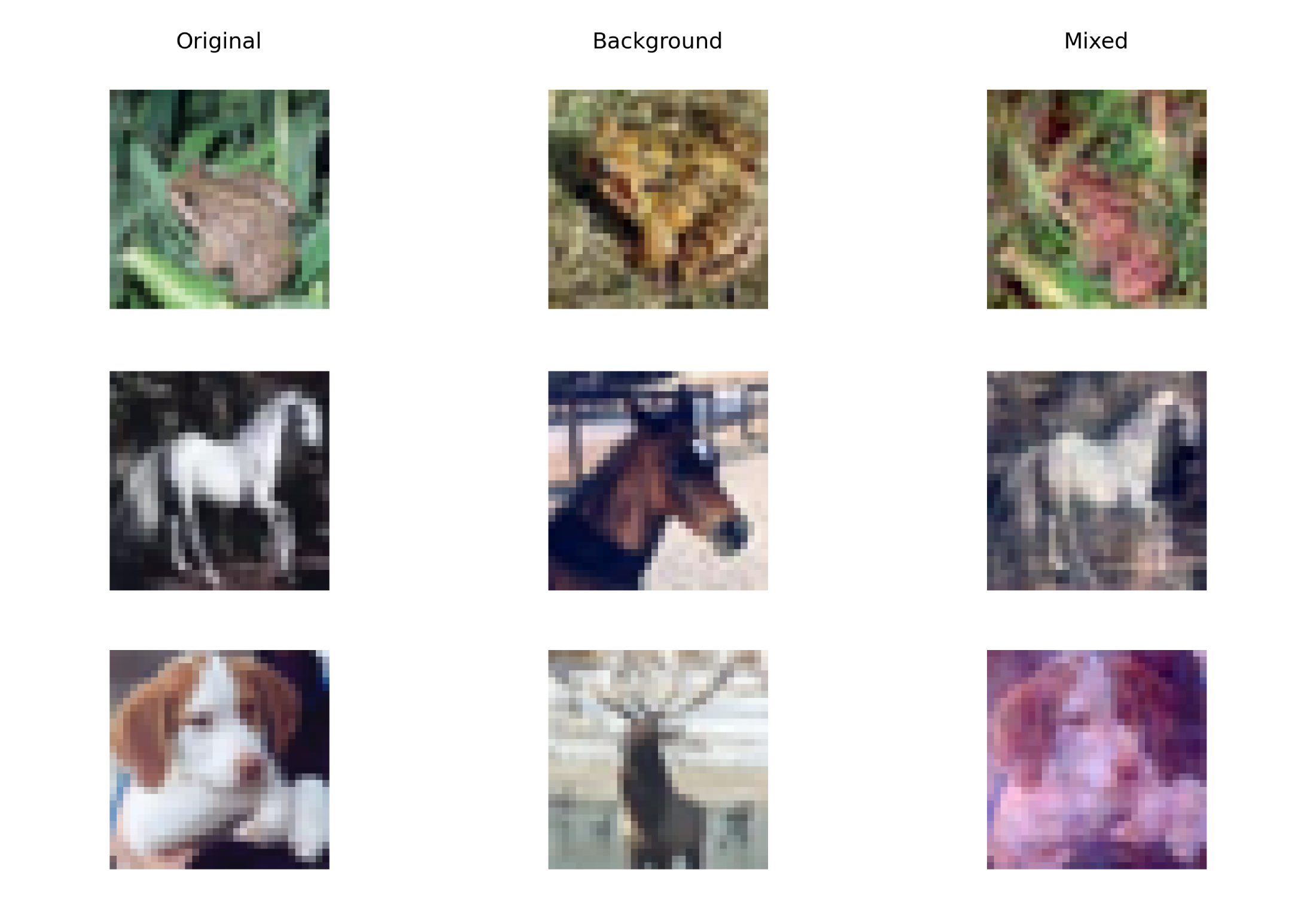}
    \caption{Visualization of the Fourier-based augmentations on CIFAR-10 with varying $\alpha$. In the first row, the foreground of the original image of a frog is placed in a different background relevant to an another frog to capture the space of the frogs better.}
    \label{fig:fft}
    \vspace{-4mm}
\end{figure}

\section{Experiments}
\subsection{Experimental Setup}

\begin{table*}[!ht]
    \centering
    \caption{Test set accuracy from linear evaluations of self-supervised learning methods (higher is better).}
    \begin{tabularx}{\textwidth}{c*{6}Y}
        \toprule
        \multirow{2}{*}{Method} & \multicolumn{2}{c}{CIFAR-10} & \multicolumn{2}{c}{CIFAR-100} & \multicolumn{2}{c}{Tiny-ImageNet} \\
        \cmidrule(lr){2-3}\cmidrule(lr){4-5}\cmidrule(lr){6-7}
        & Top-1 & Top5 & Top-1 & Top-5 & Top-1 & Top-5 \\
        \midrule
        SimCLR & 93.04\% & 99.82\% & 67.90\% & 91.24\% & 42.39\% & 69.99\% \\
        Debiased & 90.67\% & 99.71\% & 64.89\% & 89.31\% & 45.01\% & 71.03\% \\
        Hard Negative & 89.09\% & 99.50\% & 61.43\% & 86.17\% & 44.84\% & 70.73\% \\
        SSA-CLR & \textbf{94.08\%} & \textbf{99.90\%} & \textbf{71.90\%} & \textbf{92.93\%} & \textbf{46.87\%} & \textbf{72.72\%} \\
        \bottomrule
    \end{tabularx}
    \label{tab:linear}
    \vspace{-2mm}
\end{table*}

For empirical evaluations, we conduct experiments on various self-supervised learning methods on different datasets for a comprehensive look at comparative results. We opt for using ResNet-50 as the architecture of choice for our feature extractor, and a 2-layer projection head similar to other works in the field \cite{chen2020simple,chuang2020debiased,robinson2020hard}. For comparison, we evaluate our results with similar works aiming to improve SimCLR's baseline through debiasing training data \cite{chuang2020debiased}, and selectively training the model on hard negative data \cite{robinson2020hard}, using their official open-sourced codebase \footnote{\url{https://github.com/chingyaoc/DCL}} \footnote{\url{https://github.com/joshr17/HCL}} and hyperparameter sets as mentioned in the original papers. Our experiments are also seeded appropriately for reproducibility. 

We note that we do not seek state-of-the-art performances in our experiments. Alternatively, we want to demonstrate the usefulness of sharpness \& shift-aware components by comparing our SSA-CLR with SimCLR and other relevant baselines aiming to improve the InfoNCE loss of SimCLR. 

\subsection{Linear Evaluation Performance}

We conduct our experiment on three datasets of increasing difficulty: CIFAR-10 being a standard evaluation dataset, CIFAR-100 for a harder small-size problem, and Tiny-ImageNet for a computationally-feasible real world problem. The complete results are listed in \autoref{tab:linear}.

\vspace{0mm}
\textbf{CIFAR-10 \& CIFAR-100.}
For these datasets, experiments are run with batch size 256 for 1000 epochs. Evaluation results show that our method outperforms the standard SimCLR baselines by 1.04\% and 4\% in top-1 accuracy for CIFAR-10 and CIFAR-100, respectively. Regarding top-5 accuracy, SSA-CLR yields a difference of +0.08\% on CIFAR-10 which may be credited to randomness, and a notable +1.69\% on the harder CIFAR-100, where the higher difficulty highlights the difference between the two methods' performances. Surprisingly, \cite{chuang2020debiased} and \cite{robinson2020hard} both yield worse results than the which of the baseline, suggesting that these methods are heavily reliant on hyperparameter tuning and initialization randomness. 
\vspace{0mm}

\textbf{Tiny-ImageNet.}
For this dataset, due to our lack of access to powerful hardware, we will only run these methods for 500 epochs with batch size 64.
In this more practical case, both \cite{chuang2020debiased} and \cite{robinson2020hard} yield a noticeable improvement over the baseline, netting over 2\% in accuracy. Consistently and strangely, \cite{robinson2020hard} still gives a lower number than which of \cite{chuang2020debiased}, while the former is supposed to be a direct improvement over the latter. Overall, SSA-CLR still outperforms all other method.

\subsection{Feature Robustness}
\label{subsec:robust}
We also evaluate all methods' robustness by measuring their resulting models' robust accuracy using the Fast Gradient Sign Method (FGSM) attack \cite{43405} with perturbation budget $\epsilon=8/255$. In all experiments, our method achieve the highest robust accuracy amongst other self-supervised learning methods for free (i.e. without any adversarial training). The full results are listed in \autoref{tab:fgsm}.

\begin{table}[t]
    \centering
    \caption{Test set robust accuracy from linear evaluations of SSL methods with FGSM attack (higher is better).}
    \vspace{-2mm}
    \begin{tabularx}{\linewidth}{*{4}c}
        \toprule
        Dataset & {CIFAR-10} & {CIFAR-100} & Tiny-ImageNet \\
        \midrule
        SimCLR & 64.15\% & 28.94\% & 10.00\% \\
        Debiased & 51.37\% & 16.56\% & 9.94\% \\
        Hard Neg. & 55.04\% & 19.68\% & 11.89\% \\
        SSA-CLR & \textbf{69.47\%} & \textbf{33.94\%} & \textbf{12.83\%} \\
        \bottomrule
    \end{tabularx}
    \label{tab:fgsm}
    \vspace{-0mm}
\end{table}

We analyze the previous phenomenon by plotting t-SNE visualization of the extracted features in \autoref{fig:tsne}. The features from SSA-CLR are much more well-formed comparing to which from SimCLR, with class clusters' being easier to separate, and hard-to-classify points being much closer to their true classes' clusters. This leads to our downstream decision boundary being more robust to both generalization errors and adversarial attacks.

\begin{figure}[t]
    \centering
    \begin{subfigure}{0.49\linewidth}
        \includegraphics[width=\linewidth]{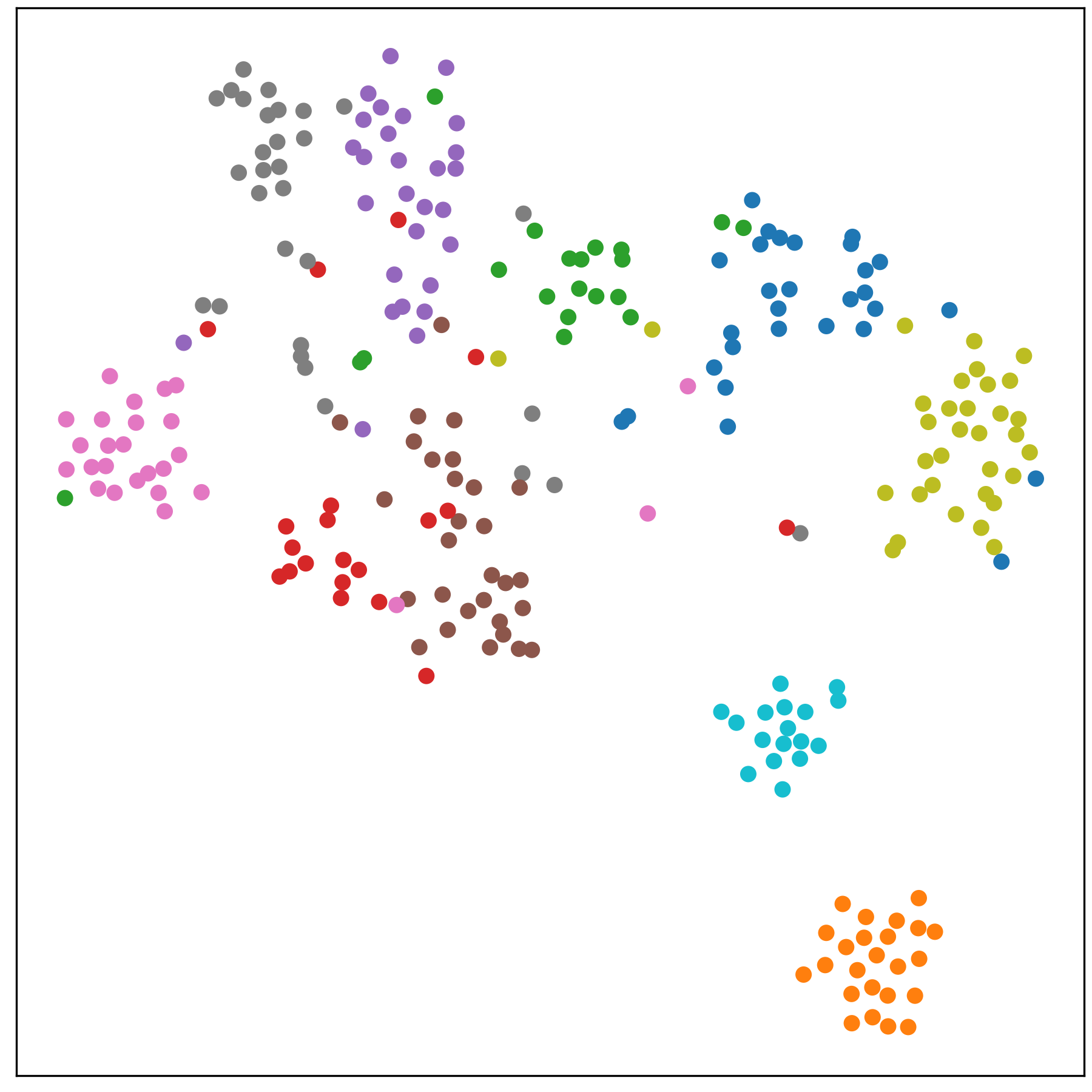}
        \caption{SimCLR}
    \end{subfigure}
    \hfill
    \begin{subfigure}{0.49\linewidth}
        \includegraphics[width=\linewidth, angle=180]{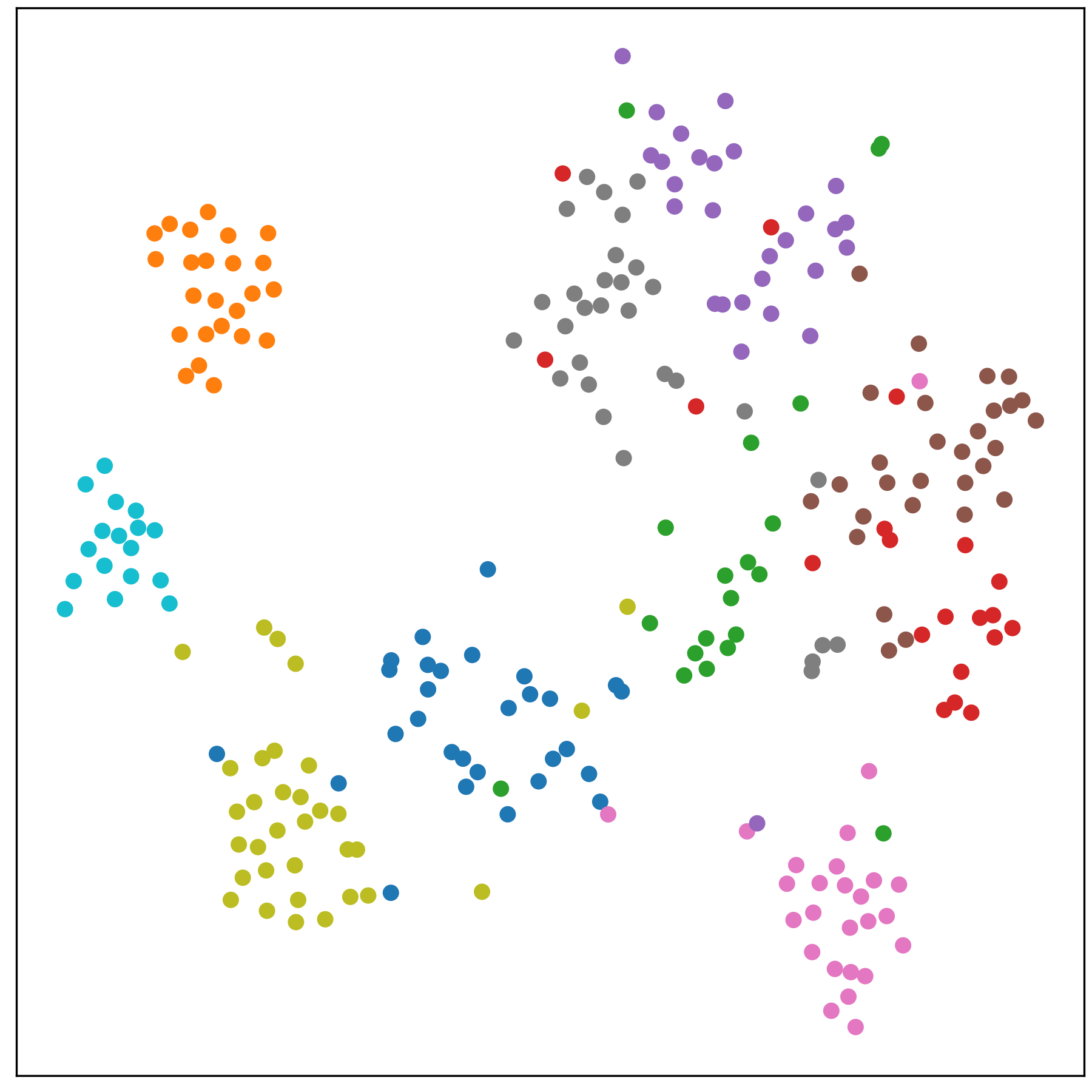}
        \caption{SSA-CLR}
    \end{subfigure}
    \vspace{-2mm}
    \caption{t-SNE visualization of the learned features on CIFAR-10.}
    \label{fig:tsne}
    \vspace{-5mm}
\end{figure}

\section{Ablation Studies}
\subsection{Sharpness-Aware Approaches}
\label{sec:sam+augs}
\begin{figure*}[!ht]
    \centering
    \begin{subfigure}{0.49\textwidth}
        \centering
        \includegraphics[width=0.85\linewidth]{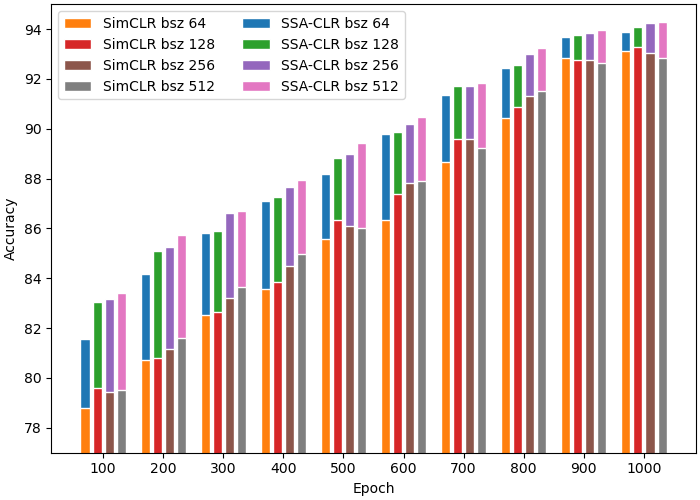}
        \caption{Standard accuracy}
    \end{subfigure}
    \hfill
    \begin{subfigure}{0.49\textwidth}
        \centering
        \includegraphics[width=0.85\linewidth]{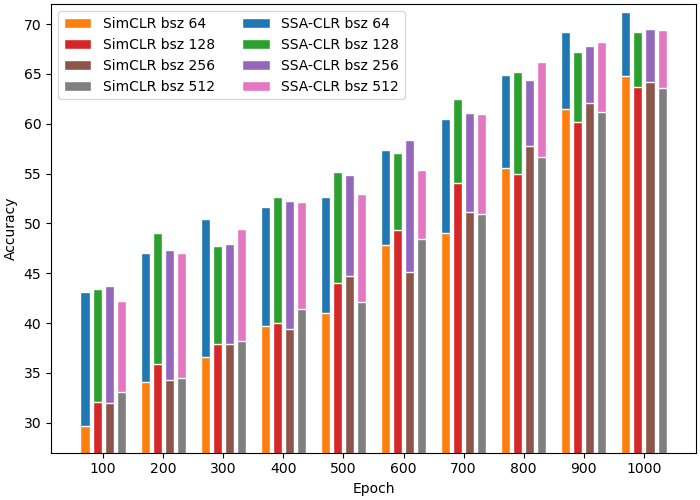}
        \caption{Robust accuracy}
    \end{subfigure}
    \vspace{-2mm}
    \caption{Test accuracy of models under different batch sizes and numbers of training epochs on CIFAR-10 (higher is better).}
    \label{fig:bsz-saclr}
    \vspace{-4mm}
\end{figure*}

In this section, we evaluate the performance of various sharpness-aware approaches to self-supervised learning methods. All experiments in this section are run with batch size 256 for 500 epochs. With vanilla SimCLR as a baseline, we compare test set accuracy of these methods using the linear evaluation procedure. For these experiments, we try applying Sharpness-Aware Minimization \cite{foret2021sharpnessaware,kwon2021asam} with different configurations, and FFT augmentations with different settings and hyperparameters. Specifically, for SAM we use $\rho=0.05$ as recommended in \cite{foret2021sharpnessaware}; and for Adaptive SAM $\rho=2.0$ based on the tuning results from \cite{phan2022improving}.

To combat distribution shift, we try our aforementioned approaches of utilizing frequency-domain augmentations \cite{Xu_2021_CVPR}, replacing the previous augmentations with our new FFT augmentations. For this experiment, we try different hyperparameter values $\alpha\in\{0.1, 0.2, 0.5, 1.0\}$. We report the obtained results in \autoref{tab:fft}.

Enforcing smoothness in the self-supervised learning process gives us an improvement in performance, with Adaptive SAM increasing more than 1\%. Adding FFT augmentations, we can increase a further 0.31\% using the default $\alpha=0.2$ hyperparameter used in \cite{Xu_2021_CVPR}. As a result, we use (1) with $\alpha=0.2$ as the default for our method.

\begin{table}[!ht]
    \centering
    \vspace{-1mm}
    \caption{Test accuracy from linear evaluations of different sharpness-aware approaches on CIFAR-10.}
    \vspace{-2mm}
    \begin{tabularx}{\linewidth}{YcY}
        \toprule
        \multicolumn{2}{c}{Method}          & Accuracy \\
        \midrule
        \multicolumn{2}{c}{SimCLR}          & 92.12\%  \\
        \midrule
        \multicolumn{2}{c}{SimCLR + SAM}    & 92.68\%  \\
        \multicolumn{2}{c}{SimCLR + ASAM}   & 93.21\%  \\
        \multicolumn{2}{c}{SimCLR + FFT}    & 92.72\%  \\ 
        \midrule
        \multirow{4}{*}{SSA-CLR} & $\alpha=0.1$ & 93.24\%  \\
                             & $\alpha=0.2$ & \textbf{93.52\%}  \\
                             & $\alpha=0.5$ & 92.84\%  \\
                             & $\alpha=1.0$ & 90.40\%  \\
        \bottomrule
    \end{tabularx}
    \label{tab:fft}
    \vspace{-3mm}
\end{table}

\subsection{Bridging Distribution Shift}
As derived in Eq. (\ref{shift_1}), the gap between the real constrastive distribution and the practical augmentation-based sampling can be quantitatively measured as
{\smaller{}
\[
\sum_c\pi_{c}\mathbb{E}_{x\sim p_{c}}\left[\Vert\mathbb{E}_{x^{+}\sim p_{c}}\left[f_{\theta}\left(x^{+}\right)\right]-\mathbb{E}_{t\sim\mathcal{T},x^{+}=t\left(x\right)}\left[f_{\theta}\left(x^{+}\right)\right]\Vert^{\frac{1}{2}}\right] 
\]
}

We thus proceed to evaluate this term to verify the effectiveness of our augmentation approach in combating this distribution shift, which experiment results can be found in \autoref{tab:dshift}. As expected, this gap is smaller when we apply our FFT augmentations onto data.

\begin{table}[t]
    \centering
    \caption{Distribution shift gap of self-supervised learning methods on CIFAR-10 (lower is better).}
    \vspace{-2mm}
    \begin{tabularx}{\linewidth}{cYY}
        \toprule
        Augmentation & Original & FFT \\
        \midrule
        Gap & 0.856 & 0.848 \\
        \bottomrule
    \end{tabularx}
    \label{tab:dshift}
    \vspace{-5mm}
\end{table}

\subsection{Batch Sizes and Epochs}


We compare our method with SimCLR baseline on batch sizes ranging from 64 to 512, checkpointing at every 100th epoch, by measuring the model's performance with linear evaluation. As we can see, models trained with larger batch sizes generally yield better performance, where they all start to plateau at around 1000 epochs. Moreover, the gap between SSA-CLR and SimCLR narrows as epoch count increases, suggesting that our method converges much faster.
These same phenomena can also be observed in our robust evaluation using the same configurations listed in Section \ref{subsec:robust}.
The full results are plotted in \autoref{fig:bsz-saclr}.

\section{Conclusions}
In this work, we introduce Sharpness \& Shift-Aware Contrastive Learning, where we aim to improve self-supervised learning by enforcing flatness of the feature's extractor; and bridging the gap of between sampling from the ideal contrasting distribution and augmentation-based methods currently being in use. Our theoretical development shows that this loss surface's flatness lets us bound our linear evaluation loss by our contrastive learning loss, guaranteeing downstream performance. For future works, one may be interested in discovering other means of combatting this sampling distribution shift with different augmentations and/or methods.

\clearpage

{\small
\bibliographystyle{ieee_fullname}
\bibliography{ref}
}

\onecolumn
\renewcommand\thesection{\Alph{section}}
\renewcommand\thesubsection{\thesection.\arabic{subsection}}
\setcounter{section}{0}
\section*{Appendix}
\section{Proofs}
We present the proofs of theorems of the main paper.
\begin{thm} \label{thm:1}
Assume the given model, the following inequality holds
\begin{align*}
    \Lc_{\Dc_{\mathrm{sup}}}^{\mathrm{sup}}(\theta) \leq \Lc_{\Dc_{\mathrm{un}}}^{\mathrm{un}}(\theta, p_{\mathrm{pos}}) - O\Big(\frac{1}{\sqrt{K}} \Big) -\log(\beta)- O\Big(\frac{1}{\beta}\Big).
\end{align*}
\end{thm}
\begin{proof}
The proof has three steps as follows: 
\begin{itemize}
    \item To show that there exists an weight matrix $\bW$ such that $\Lc_{\Dc_{\supp}}^{\supp}(\theta,\bW)\leq \overline{\Lc}_{\Dc_{\text{un}}}^{\text{un}}(\theta) $.
    \item To bound the difference between $\mathcal{L}_{\mathcal{D}_{\mathrm{un}}}^{\text{un}}(\theta) $ and $\overline{\mathcal{L}}_{\mathcal{D}_{\text{un}}}^{\text{un}}(\theta)$.
    \item To show the inequality between $\Lc_{\Dc_{\un}}^{\mathrm{un}}(\theta)$ and $\Lc_{\Dc_{\supp}}^{\supp}$.
\end{itemize}
\textbf{First step:} We choose 
$\bW_{c}=\mathbb{E}_{x\sim p_{c}}\left[f_{\theta}\left(x\right)\right],\forall c\in\mathcal{Y}$ and $\bW = \big[\bW_c \big]_{c=1}^C$
\begin{align}
\mathcal{L}_{\mathcal{D}_{\text{sup}}}^{\text{sup}}\left(\theta,\bar{W}\right) & =\mathbb{E}_{\left(x,y\right)\sim\mathcal{D}_{\text{sup}}}\left[\tau_{\text{CE}}\left(\bar{W}\cdot f_{\theta}\left(x\right),y\right)\right]=\sum_{c=1}^{M}\pi_{c}\mathbb{E}_{x\sim p_{c}}\left[\tau_{\text{CE}}\left(\bar{W}\cdot f_{\theta}\left(x\right),c\right)\right]\nonumber \\
&=  -\sum_{c=1}^{M}\pi_{c}\mathbb{E}_{x\sim p_{c}}\left[\log\frac{\exp\left\{ \bar{W}_{c}\cdot f_{\theta}\left(x\right)/\tau\right\} }{\sum_{y=1}^{M}\exp\left\{ \bar{W}_{y}\cdot f_{\theta}\left(x\right)/\tau\right\} }\right]\nonumber \\
&=  -\sum_{c=1}^{M}\pi_{c}\mathbb{E}_{x\sim p_{c}}\left[\frac{\bar{W}_{c}\cdot f_{\theta}\left(x\right)}{\tau}-\log\left(\sum_{y=1}^{M}\exp\left\{ \frac{\bar{W}_{y}\cdot f_{\theta}\left(x\right)}{\tau}\right\} \right)\right]\nonumber \\
&=-  \sum_{c=1}^{M}\frac{\pi_{c}}{\tau}\bar{W}_{c}\cdot \mathbb{E}_{x\sim p_{c}}\left[f_{\theta}\left(x\right)\right]+\sum_{c=1}^{M}\pi_{c}\mathbb{E}_{x\sim p_{c}}\left[\log\left(\sum_{y=1}^{M}\exp\left\{ \frac{\bar{W}_{y}.f_{\theta}\left(x\right)}{\tau}\right\} \right)\right]\nonumber \\
&=  -\sum_{c=1}^{M}\frac{\pi_{c}}{\tau}\norm{\bar{W}_{c}}^{2}+\sum_{c=1}^{M}\pi_{c}\mathbb{E}_{x\sim p_{c}}\left[\log\left(\sum_{y=1}^{M}\exp\left\{ \frac{\bar{W}_{y}\cdot f_{\theta}\left(x\right)}{\tau}\right\} \right)\right].\label{eq:L_sup}
\end{align}

Define 
$$
\overline{\mathcal{L}}_{\mathcal{D}_{un}}^{un}\left(\theta,p_{\text{pos}}\right)=\mathbb{E}_{\left(x,x^{+}\right)\sim p_{\text{pos}}}\left[-\frac{1}{\tau}f_{\theta}\left(x\right).f_{\theta}\left(x^{+}\right)\right]+\mathbb{E}_{x\sim p_{\text{data}}}\left[\log\mathbb{E}_{x^{-}\sim p_{\text{data}}}\left[\exp\left\{ \frac{f_{\theta}\left(x\right).f_{\theta}\left(x^{-}\right)}{\tau}\right\} \right]\right],
$$
we then show an lower bound for $\overline{\Lc}_{\Dc_{\un}}^{\un}(\theta, p_{pos})$:
\begin{align}
\overline{\mathcal{L}}_{\mathcal{D}_{\text{un}}}^{\text{un}}\left(\theta, p_{pos}\right) & =\mathbb{E}_{\left(x,x^{+}\right)\sim p_{\text{pos}}}\left[-\frac{1}{\tau}f_{\theta}\left(x\right)\cdot f_{\theta}\left(x^{+}\right)\right]+\mathbb{E}_{x\sim p_{\text{data}}}\left[\log\mathbb{E}_{x^{-}\sim p_{\text{data}}}\left[\exp\left\{ \frac{f_{\theta}\left(x\right) \cdot f_{\theta}\left(x^{-}\right)}{\tau}\right\} \right]\right]\nonumber \\
&=  \mathbb{E}_{\left(x,x^{+}\right)\sim p_{\text{pos}}}\left[-\frac{1}{\tau}f_{\theta}\left(x\right) \cdot f_{\theta}\left(x^{+}\right)\right]+\mathbb{E}_{x\sim p_{\text{data}}}\left[\log\left(\sum_{c=1}^{M}\pi_{c}\mathbb{E}_{x^{-}\sim p_{c}}\left[\exp\left\{ \frac{f_{\theta}\left(x\right) \cdot f_{\theta}\left(x^{-}\right)}{\tau}\right\} \right]\right)\right]\nonumber \\
&\geq  \mathbb{E}_{\left(x,x^{+}\right)\sim p_{\text{pos}}}\left[-\frac{1}{\tau}f_{\theta}\left(x\right) \cdot f_{\theta}\left(x^{+}\right)\right]+\mathbb{E}_{x\sim p_{\text{data}}}\left[\log\left(\sum_{c=1}^{M}\pi_{c}\exp\left\{ \frac{f_{\theta}\left(x\right)\cdot \mathbb{E}_{x^{-}\sim p_{c}}\left[f_{\theta}\left(x^{-}\right)\right]}{\tau}\right\} \right)\right]\nonumber \\
&=  -\frac{1}{\tau}\sum_{c=1}^{M}\pi_{c}\mathbb{E}_{x,x^{+}\sim p_{c}}\left[f_{\theta}\left(x\right)\cdot f_{\theta}\left(x^{+}\right)\right]+\mathbb{E}_{x\sim p_{\text{data}}}\left[\log\left(\sum_{c=1}^{M}\pi_{c}\exp\left\{ \frac{W_{c}\cdot f_{\theta}\left(x\right)}{\tau}\right\} \right)\right]\nonumber \\
&=  -\frac{1}{\tau}\sum_{c=1}^{M}\pi_{c}\mathbb{E}_{x\sim p_{c}}\left[f_{\theta}\left(x\right)\right]\cdot \mathbb{E}_{x^{+}\sim p_{c}}\left[f_{\theta}\left(x^{+}\right)\right]+\mathbb{E}_{x\sim p_{\text{data}}}\left[\log\left(\sum_{c=1}^{M}\pi_{c}\exp\left\{ \frac{W_{c}\cdot f_{\theta}\left(x\right)}{\tau}\right\} \right)\right]\nonumber \\
&=  -\sum_{c=1}^{M}\frac{\pi_{c}}{\tau}\norm{\bar{W}_{c}}^{2}+\sum_{c=1}^{M}\pi_{c}\mathbb{E}_{x\sim p_{c}}\left[\log\left(\sum_{y=1}^{M}\exp\left\{ \frac{\bar{W}_{y}\cdot f_{\theta}\left(x\right)}{\tau}\right\} \right)\right].\label{eq:L_un}
\end{align}
Combining  \eqref{eq:L_sup} and \eqref{eq:L_un} , we get
\begin{align*}
    \Lc_{\Dc_{\supp}}^{\supp}(\theta,\bW) \leq \overline{\Lc}_{\Dc_{\un}}^{\un}(\theta, p_{\mathrm{pos}}).
\end{align*}
\textbf{Second step:} We start with decomposing the $\mathcal{L}^{\text{un}}_{\mathcal{D}_{\text{un}}}(\theta, p_{\mathrm{pos}})$
\begin{align*}
\mathcal{L}^{\text{un}}_{\mathcal{D}_{\text{un}}}(\theta, p_{pos}) &= \mathbb{E}_{(x,x^+)\sim p_{\pos},x_{1:K}^{-}\sim p_{\dat}}\left[ -\log \frac{\exp\big\{\frac{ f_{\theta}(x) \cdot f_{\theta}(x^+)}{\tau} \big\}}
{\exp\big\{\frac{f_{\theta}(x) \cdot f_{\theta}(x^+)}{\tau} \big\} + \frac{\beta}{K} \sum_{k=1}^K \exp\big\{\frac{f_{\theta}(x) \cdot f_{\theta}(x_k^{-})}{\tau} \big\}}\right] \\
&= \mathbb{E}_{(x,x^+)\sim p_{\pos},x_{1:K}^{-}\sim p_{\dat}} \Big[ -\frac{f_{\theta}(x) \cdot f_{\theta}(x^+)}{\tau} +\log\Big(\exp\Big\{\frac{f_{\theta}(x) \cdot f_{\theta}(x^+)}{\tau} \Big\} + \frac{\beta}{K} \sum_{k=1}^K \exp\Big\{\frac{f_{\theta}(x) \cdot f_{\theta}(x_k^{-})}{\tau} \Big\}\Big) \Big].
\end{align*}
The first term of $\overline{\mathcal{L}}_{\mathcal{D}_{\text{un}}}^{\text{un}}(f_{\theta})$ is the same as the first term of $\mathcal{L}_{\mathcal{D}_{\text{un}}}^{\text{un}}(f_{\theta})$. Thus, we only have to deal with the second terms of both quantities. We have
\begin{align*}
    \mathbb{E}_{x\sim p_{\dat}}\Big[\log \Big(\mathbb{E}_{x^{-}\sim p_{\dat}}\Big[\exp\big\{\frac{f_{\theta}(x) \cdot f_{\theta}(x^{-})}{\tau}\big\}\Big]\Big)\Big] + \log \beta = \mathbb{E}_{x\sim p_{\pos}}\Big[\log\Big( \beta \mathbb{E}_{x^{-}\sim p_{\dat}}\Big[\exp\big\{\frac{f_{\theta}(x) \cdot f_{\theta}(x^{-})}{\tau}\big\} \Big]\Big) \Big]
\end{align*}
Therefore,
\begin{align*}
\overline{\mathcal{L}}_{\mathcal{D}_{\text{un}}}^{\text{un}}(\theta, p_{pos}) - \mathcal{L}_{\mathcal{D}_{\text{un}}}^{\text{un}}(\theta, p_{pos}) &= \mathbb{E}_{(x,x^+)\sim p_{\pos},x_{1:K}^{-}\sim p_{\dat}} \Bigg[\log\Big(\beta\mathbb{E}_{x^{-}}\Big[\exp\big\{ \frac{f_{\theta}(x) \cdot f_{\theta}(x^{-})}{\tau}\big\}\Big] \Big) - \\
& \qquad \log\Big(\exp\Big\{\frac{f_{\theta}(x) \cdot f_{\theta}(x^+)}{\tau} \Big\} + \frac{\beta}{K} \sum_{k=1}^K \exp\Big\{\frac{f_{\theta}(x) \cdot f_{\theta}(x_k^{-})}{\tau} \Big\} \Big) \Bigg]
\end{align*}
Denote 
\begin{align*}
    \frac{1}{K}\sum_{k=1}^K\exp\Big\{\frac{f_{\theta}(x)\cdot f_{\theta}(x_k^{-})}{\tau} \Big\} - \mathbb{E}_{x^{-}} \Big[\exp\Big\{\frac{f_{\theta}(x) \cdot f_{\theta}(x^{-})}{\tau} \Big\} \Big]&:= Y_n\\
    \mathbb{E}_{x^{-}} \Big[\exp\Big\{\frac{f_{\theta}(x) \cdot f_{\theta}(x^{-})}{\tau} \Big\} \Big]&:= \alpha_x\\
    \exp\Big\{\frac{f_{\theta}(x)\cdot f_{\theta}(x^+)}{\tau} \Big\}&:= Z.
\end{align*}
The inner part of the expectation $\mathbb{E}_{(x,x^+)\sim p_{\pos},x_{1:K}^{-}\sim p_{\dat}}$ is written as
\begin{align*}
    \log\big( \beta \alpha_x\big) - \log\big(\beta(Y_n+ \alpha_x) + Z\big) &= -\log\frac{\beta Y_n + \beta \alpha_x + Z}{\beta \alpha_x} = -\log\Big(1 +\frac{Z}{\beta \alpha_x} + \frac{ Y_n}{\alpha_x} \Big) \\
    &= -\log\big(1 + \frac{Z}{\beta\alpha_x} \big) - \log\Big(1 + \frac{Y_n/\alpha_x}{1 + Z/(\beta \alpha_x)}\Big)
\end{align*}
Thus, we get
\begin{align*}
    \E_{(x,x^+)\sim p_{\pos}, x_{1:K}^{-}\sim p_{\dat}}\Big[-\log\big(1 + \frac{Z}{\beta\alpha_x}\big) -\log \Big(1 + \frac{Y_n/\alpha_x}{1 + Z/(\beta \alpha_x)} \Big)\Big] =
    \E_{(x,x^+)\sim p_{\pos}}\Big[-\log\big(1+ \frac{Z}{\beta\alpha_x} \big) \Big] + \\
    \E_{(x,x^+)\sim p_{\pos}, x_{1:K}^{-}\sim p_{\dat}}\Big[-\log\Big(1 + \frac{Y_n/\alpha_x}{1+Z/(\beta\alpha_x)} \Big) \Big].
\end{align*}
We also have $\exp(-1/\tau)\leq \exp\Big\{\frac{f_{\theta}(x) \cdot f_{\theta}(x^{-})}{\tau} \Big\}\leq \exp(1/\tau)$. It follows that $\exp(-1/\tau)\leq Z,\alpha_x \leq \exp(1/\tau)$. Then we deduce bounds for other quantities 
\begin{align*}
    &\frac{1}{\beta}\exp(-2/\tau)\leq \frac{Z}{\beta\alpha_x} \leq \frac{1}{\beta}\exp(1/\tau) \exp(1/\tau) = \frac{1}{\beta } \exp(2/\tau)\\
    &\exp(-2/\tau)\leq \frac{Y_n+\alpha_x}{\alpha_x} \leq \exp(2/\tau)\\
    & \exp(-2/\tau)-1\leq \frac{Y_n}{\alpha_x} \leq \exp(2/\tau) - 1\\
    & \exp(-2/\tau) - 1\leq \frac{Y_n}{\alpha_x} + \frac{Z}{\beta\alpha_x} \leq \frac{1}{\beta}\exp(2/\tau) + \exp(2/\tau) - 1\\
    &\frac{Y_n/\alpha_x}{1+ Z/(\beta\alpha_x)} \geq \frac{\exp(-2/\tau) -1}{1+\frac{1}{\beta}\exp(-2/\tau)}
\end{align*}
In both terms inside the expectation $\E_{(x,x^+)\sim p_{\pos},x_{1:K}^{-}\sim p_{\dat}}$, we have the form $\log(1+t)$.  

\noindent For $t>0$, we use the inequality 
\begin{align*}
    \log(1+t) \leq t.
\end{align*}
For $-1< t <0$ we have
\begin{align*}
    \big|\log(1+t) \big| &= \big|\log\frac{1}{1-|t|}\big|
    = \big| \log(1 + |t| + |t|^2 + \ldots )\big| \leq |t| + |t|^2 + \ldots + |t|^m + \ldots = |t| \frac{1}{1-|t|}.
\end{align*}
Use these inequalities for $\log(1+t)$, we get
\begin{align*}
\Big|-\log\Big(1+ \frac{Y_n/\alpha_x}{1+ Z/(\beta\alpha_x)}\Big) \Big| &\leq \Big|\frac{Y_n/\alpha_x}{1+ Z/(\beta\alpha_x)}\Big|\times  \frac{1}{1+ \frac{\exp(-2/\tau) -1}{1+\frac{1}{\beta}\exp(-2/\tau)}} = \Big|\frac{Y_n/\alpha_x}{1+ Z/(\beta\alpha_x)}\Big| \times \frac{1+ \frac{1}{\beta}\exp(-2/\tau)}{(1+1/\beta)\exp(-2/\tau)}\\
&\leq \Big|\frac{Y_n/\alpha_x}{1+ Z/(\beta\alpha_x)}\Big| \exp(2/\tau).
\end{align*}
By Cauchy-Schwarz inequality, 
\begin{align*}
    \mathbb{E}_{x_{1:K}^{-}\sim p_{\dat}}|Y_n| \leq \sqrt{\mathbb{E}_{x_{1:K}^{-}} [Y_n^2]} \leq  \sqrt{\frac{\exp(2/\tau)- \exp(-2/\tau)}{K}}. 
\end{align*}
Therefore,
\begin{align*}
   \left| \E_{(x,x^+)\sim p_{\pos},x_{1:K}^{-}\sim p_{\dat}} \Big[-\log\Big(1 + \frac{Y_n/\alpha_x}{1+ Z/(\beta\alpha_x)}\Big) \Big]  \right| \leq \E_{(x,x^+)\sim p_{\pos}} \E_{x_{1:K}\sim p_{\dat}}\Big|\frac{Y_n/\alpha_x}{1+ Z/(\beta\alpha_x)} \Big| \times \exp(2/\tau) \\
   \leq \frac{1}{\alpha_x} \exp(2/\tau) \times \sqrt{\frac{\exp(2/\tau)-\exp(-2/\tau)}{K}} = O(K^{-1/2}).
\end{align*}
For the other term inside $\E_{(x,x^+)\sim p_{\pos},x_{1:K}\sim p_{\dat}}$
\begin{align*}
    \E_{(x,x^+)\sim p_{\pos}}\Big[-\log\big(1 + \frac{Z}{\beta\alpha_x}\big) \Big]
\end{align*}
is a constant and negative, its absolute value is bounded by
\begin{align*}
    \E_{(x,x^+)\sim p_{\pos}}\Big[\log\big(1 + \frac{Z}{\beta\alpha_x}\big) \Big] \leq \E_{(x,x^+)\sim p_{\pos}} \frac{Z}{\beta \alpha_x} \leq \frac{1}{\beta}\exp(2/\tau).
\end{align*}
For a lower bound of $\log(1+t)$, we have
\begin{align*}
    e^{\frac{1}{1+t}-1} \geq \frac{1}{1+t} \Rightarrow e^{\frac{1}{t+1}} \geq \frac{e}{1+t}
\Rightarrow (1+t) e^{\frac{1}{1+t}} \geq e \Rightarrow (1+t) \geq e^{\frac{t}{1+t}} \Rightarrow \log(1+t) \geq \frac{t}{1+t}.
\end{align*}
Hence,
\begin{align*}
    \E_{(x,x+)\sim p_{\pos}}\Big[\log\big(1+\frac{Z}{\beta \alpha_x} \big) \Big] &\geq \E_{(x,x+)\sim p_{\pos}}\Big[\frac{Z/(\beta\alpha_x)}{1 + Z/(\beta\alpha_x)} \Big] \geq \E_{(x,x+)\sim p_{\pos}}\Big[ \frac{Z}{\beta \alpha_x} \Big]\times \frac{1}{1 + \frac{1}{\beta}\exp(2/\tau)} \\
    &\geq \frac{1}{\beta}\frac{\exp(-2\tau)}{1+\frac{1} {\beta}\exp(2/\tau)}. 
\end{align*}
Together, we have 
\begin{align*}
\mathcal{L}_{\mathcal{D}_{\text{un}}}^{\text{un}}(\theta, p_{pos}) -\overline{\mathcal{L}}_{\mathcal{D}_{\text{un}}}^{\text{un}}(\theta, p_{pos}) = \E_{(x,x^+)\sim p_{\pos}} \Big[\log\big(1 +  \frac{Z}{\beta \alpha_x} \big) \Big] + O\big(\frac{1}{\sqrt{K}} \big) = O\big(\frac{1}{\beta} \big) + O\big(\frac{1}{\sqrt{K}} \big).
\end{align*}
\textbf{Last step:} We have
\begin{align*}
\Lc_{\Dc_{\supp}}^{\supp}(\theta,\bW) \geq \inf_{W}\Lc_{\Dc_{\supp}}^{\supp}(\theta,W) = \Lc_{\Dc_{\supp}}^{\supp}(\theta).
\end{align*}
It follows that
\begin{align*}
    \Lc_{\Dc_{\supp}}^{\supp}(\theta) \leq \Lc_{\Dc_{\unm}}^{\unm}(\theta, p_{pos}) - O\Big(\frac{1}{\sqrt{K}} \Big) -\log(\beta)- O\Big(\frac{1}{\beta}\Big).
\end{align*}

\end{proof}

\begin{thm} \label{thm:2}
    For $0< \delta < 1$, with probability $1-\delta$ over the random choice of $\Sc \sim \Dc_{\unm}^N$, we have
    \begin{align*}
        \Lc_{\Dc_{\unm}}^{\unm}(\theta,p_{\mathrm{pos}}) \leq \max_{\theta^{\prime}: \|\theta^{\prime}- \theta\|< \rho} \Lc_{\Sc}^{\unm}(\theta,p_{\mathrm{pos}}) + \frac{1}{\sqrt{N}}\Bigg[ T\log\Big(1+\frac{\|\theta\|^2}{T\sigma^2} \Big) + \log\frac{1}{\delta} + O\Big(\log(N+T)\Big) + \frac{L^2}{8} + 2L \Bigg]
    \end{align*}
    under the assumptions that for any $\sigma >0$, $\Lc_{\Dc_{\mathrm{un}}}^{\mathrm{un}}(\theta) \leq \E_{\theta^{\prime}\sim \Nc(\theta,\sigma^2 \mathbb{I})}$  where $L = \frac{2}{\tau} + \log(1+\beta) $ and $\rho = \sigma\big(\sqrt{T} +\sqrt{\log(N)} \big)$.
\end{thm}
\begin{proof}
    Given $z = \Big[x,x^+,\big[x_k^{-} \big]_{k=1}^K \Big]$ where $(x,x^+)\sim p_{\pos}$, $x_{1:K}^{-}\overset{\text{iid}}{\sim} p_{\dat}$. We recall the loss function is
    \begin{align*}
        \ell(f_{\theta}(x)) &= -\log \frac{\exp\Big\{\frac{f_{\theta}(x) \cdot f_{\theta}(x^{+})}{\tau} \Big\}}{\exp\Big\{\frac{f_{\theta}(x) \cdot f_{\theta}(x^{+})}{\tau} \Big\} + \frac{\beta}{K}\sum_{k=1}^K \exp \Big\{\frac{f_{\theta}(x) \cdot f_{\theta}(x_k^{-})}{\tau} \Big\}} \\
        &= \log \frac {\exp\Big\{\frac{f_{\theta}(x) \cdot f_{\theta}(x^{+})}{\tau} \Big\} + \frac{\beta}{K}\sum_{k=1}^K \exp \Big\{\frac{f_{\theta}(x) \cdot f_{\theta}(x_k^{-})}{\tau} \Big\}}{\exp\Big\{\frac{f_{\theta}(x) \cdot f_{\theta}(x^{+})}{\tau} \Big\}} \\
        &\leq \log\frac{e^{1/\tau}+ \beta e^{1/\tau}}{e^{-1/\tau}} \\
        &= \frac{2}{\tau} + \log(1+\beta).
    \end{align*}
We use the PAC-Bayes theory for $P= \Nc(\mathbf{0},\sigma_P^2\mathbb{I}_T)$ and $Q = \Nc(\theta,\sigma^2 \mathbb{I}_T)$ are the prior and posterior distributions, respectively.

By using the bound in \cite{JMLR:v17:15-290}, with probability at least $1-\delta$, we have
    \begin{align*}
        \Lc_{\Dc_{\unm}}^{\unm}(\theta,Q)\leq \Lc_{\Sc}^{\unm}(\theta, Q) + \frac{1}{\beta}\Big[ \mathsf{KL}(Q\|P) + \log \frac{1}{\delta} + \Psi(\beta,N)\Big],
    \end{align*}
    where we have defined
    \begin{align*}
        \Psi(\beta,N) = \log \E_{P}\E_{\Dc_{\unm}^N}\Big[ \exp\Big\{\beta \big(\Lc_{\Dc_{\unm}}^{\unm}(\theta) - \Lc_{\Sc}^{\unm}(\theta) \big) \Big\}\Big]
    \end{align*}
    Since the loss function is bounded by $L$, we have
    \begin{align*}
        \Psi(\beta,N) \leq \frac{\beta^2 L^2}{8N}.
    \end{align*}
     Thus, we get 
     \begin{align}
        \Lc_{\Dc_{\unm}}^{\unm}(\theta,Q)\leq \Lc_{\Sc}^{\unm}(\theta, Q) + \frac{1}{\beta}\Big[ \mathsf{KL}(Q\|P) + \log \frac{1}{\delta} + \frac{\beta^2L^2}{8N}\Big]. \label{ineq:DS_beta}
    \end{align}
    By Cauchy inequality,
    \begin{align}
        \frac{1}{\beta} \Big[\mathsf{KL}(Q\|P) + \log \frac{1}{\delta} + \frac{\beta^2L^2}{8N} \Big] &= \frac{1}{\beta}\Big[\mathsf{KL}(Q\|P) +\log \frac{1}{\delta} \Big]  + \frac{\beta L^2}{8N} \geq \frac{\sqrt{\mathsf{KL}(Q\|P) +\log\frac{1}{\delta}}}{\sqrt{2N}}\times L \nonumber\\
        &\geq \frac{\sqrt{\mathsf{KL}(Q\|P)}}{\sqrt{2N}}\times L. \label{bound:error_term}
    \end{align}
    Since $P$ and $Q$ are Gaussian distribution, the KL divergence between $Q$ and $P$ is equal to
    \begin{align*}
        \mathsf{KL}(Q\|P) = \frac{1}{2}\Big[\frac{T\sigma^2 +\|\theta\|^2}{\sigma_P^2} - T + T\log\frac{\sigma_P^2}{\sigma^2} \Big],
    \end{align*}
    where $T$ is the number of coordinate of $\theta$. Let us consider the KL divergence term $\mathsf{KL}(Q\|P)$ as a function of $\sigma_P^2$, then its derivative with respect to $\sigma_P^2$ is equal to
\begin{align*}
    T\frac{1}{\sigma_P^2} - \frac{T\sigma^2 + \|\theta\|^2}{\sigma_P^4} = \frac{T}{\sigma_P^2}\Big[1 - \frac{\sigma^2 + \|\theta\|^2/T}{\sigma_P^2}\Big],
\end{align*}
which is equal to zero when $\sigma_P^2 = \sigma^2 + \|\theta\|^2/T$. Thus
\begin{align*}
     \mathsf{KL}(Q\|P) \geq \frac{T}{2}\log\big(1 + \frac{\|\theta\|^2}{T\sigma^2}\big).
\end{align*}
Together with \eqref{bound:error_term}, we get
\begin{align*}
    \frac{1}{\beta} \Big[\mathsf{KL}(Q\|P) + \log \frac{1}{\delta} + \frac{\beta^2L^2}{8N} \Big] \geq \frac{L}{\sqrt{2N}}\sqrt{\frac{T}{2}\log\big( 1+\frac{\|\theta\|^2}{T\sigma^2}\big)} \geq L .
\end{align*}
when $\|\theta\|^2 \geq T\sigma^2 \Big[\exp\frac{4N}{T}-1\Big]$. Since the loss function $\ell$ is bounded by $L$, if $\|\theta\|^2 \geq T\sigma^2 \Big[\exp\frac{4N}{T}-1\Big]$, then the RHS of \eqref{ineq:DS_beta} is already greater than $L$. Therefore, we only need to consider the case that $\|\theta\|^2 \leq T\sigma^2 \Big[\exp\frac{4N}{T}-1\Big]$.

We need to specify $P$ in advance, since it is a prior distribution. However, we do not know in advance the value of $\theta$ that affect the KL divergence term. Hence, we build a family of distribution $P$ as follows:
    \begin{align*}
        \mathfrak{P} = \Big\{P_j = \Nc(\mathbf{0},\sigma_{P_j}^2\mathbb{I}_T): \sigma_{P_j}^2 = c \exp\big(\frac{1-j}{T}\big), c = \sigma^2 \big(1 + \exp\frac{4N}{T} \big), j = 1,2,\ldots\Big\}
    \end{align*}
Set $\delta_j = \frac{6\delta}{\pi^2j^2}$, the below inequality holds with probability at least $1-\delta_j$
 \begin{align*}
        \Lc_{\Dc_{\unm}}^{\unm}(\theta,Q)\leq \Lc_{\Sc}^{\unm}(\theta, Q) + \frac{1}{\beta}\Big[ \mathsf{KL}(Q\|P_j) + \log \frac{1}{\delta_j} + \frac{\beta^2L^2}{8N}\Big].
    \end{align*}
Thus, with probability $1-\delta$ the above inequalities hold  for all $P_j$.  We choose
\begin{align*}
    j^* =  \left \lfloor 1 + T\log\left(\frac{\sigma^2\big(1+\exp\{ 4N/T\} \big)}{\sigma^2 +\|\theta\|^2/T}\right) \right\rfloor. 
\end{align*}
Since $\frac{\|\theta\|^2}{T} \leq \sigma^2 \big[\exp\frac{4N}{T} -1 \big]$, we get $\sigma^2 + \frac{\|\theta\|^2}{T} \leq \sigma^2 \exp\frac{4N}{T} $, thus $j^*$ is well-defined.  We also have
\begin{align*}
   &T \log\frac{c}{\sigma^2+ \|\theta\|^2/T }\leq j^* \leq 1 + T\log \frac{c}{\sigma^2 + \|\theta\|^2/T} \\
   \Rightarrow\quad & \log\frac{c}{\sigma^2+ \|\theta\|^2/T } \leq \frac{j^*}{T} \leq \frac{1}{T} + \log \frac{c}{\sigma^2 + \|\theta\|^2/T}\\
   \Rightarrow \quad & -\frac{1}{T} + \log \frac{\sigma^2 + \|\theta\|^2/T}{c} \leq \frac{-j^*}{T}\leq \log\frac{\sigma^2 + \|\theta\|^2/T}{c} \\
   \Rightarrow \quad & e^{-1/T} \frac{\sigma^2+ \|\theta\|^2/T}{c} \leq e^{-j^*/T} \leq \frac{\sigma^2 + \|\theta\|^2/T}{c} \\
   \Rightarrow \quad & \sigma^2 + \frac{\|\theta\|^2}{T} \leq c e^{\frac{1-j^*}{T}} \leq e^{\frac{1}{T}}\Big(\sigma^2 + \frac{\|\theta\|^2}{T} \Big)\\
   \Rightarrow \quad &  \sigma^2 + \frac{\|\theta\|^2}{T} \leq \sigma_{P_{j^*}}^2 \leq e^{\frac{1}{T}}\Big(\sigma^2 + \frac{\|\theta\|^2}{T}\Big).
\end{align*}
 Hence, we have
\begin{align*}
    \mathsf{KL}(Q\|P_{j^*}) &=\frac{1}{2}\Big[\frac{T\sigma^2 +\|\theta\|^2}{\sigma_{P_{j^*}}^2} - T + T\log\frac{\sigma_{P_{}j^*}^2}{\sigma^2} \Big] \leq \frac{1}{2}\Big[\frac{T\sigma^2 +\|\theta\|^2}{\sigma^2 + \|\theta\|^2/T} - T + T\log\frac{e^{1/T}\big(\sigma^2 + \|\theta\|^2/T \big)}{\sigma^2}  \Big] \\
    &\leq  \frac{1}{2}\Big[1+ T\log\big(1 + \frac{\|\theta\|^2}{T\sigma^2}\big) \Big].
\end{align*}
For the term $\log\frac{1}{\delta_{j^*}}$, use the inequality $\log(1+e^t) \leq 1 + t$ for $t>0$,
\begin{align*}
    \log\frac{1}{\delta_{j^*}} &= \log \frac{(j^*)^2\pi^2}{6\delta}  = \log\frac{1}{\delta}  + \log\Big(\frac{\pi^2}{6}\Big) + 2\log(j^*) \\
    &\leq \log\frac{1}{\delta} + \log\frac{\pi^2}{6} + 2\log \Big( 1+T\log\frac{\sigma^2\big(1+ \exp(4N/T)\big)}{\sigma^2 + \|\theta\|^2/T}\Big)  \\
    &\leq \log\frac{1}{\delta} + \log\frac{\pi^2}{6} + 2\log\Big(1+ T\log\big(1+\exp(4N/T)\big)\Big) \\
    &\leq \log\frac{1}{\delta} + \log\frac{\pi^2}{6} + 2\log\Big(1+ T\big(1+\frac{4N}{T} \big) \Big) \\
    &\leq \log\frac{1}{\delta} + \log\frac{\pi^2}{6} + \log(1+T + 4N).
\end{align*}
Choosing $\beta = \sqrt{N}$, with probability at least $1-\delta$ we get
\begin{align*}
    \frac{1}{\beta} \Big[\mathsf{KL}(Q\|P_{j^*}) + \log \frac{1}{\delta_{j^*}} + \frac{\beta^2 L^2}{8N} \Big] \leq \frac{1}{\sqrt{N}}\Big[\frac{1}{2}+\frac{T}{2}\log\Big(1 +\frac{\|\theta\|^2}{T\sigma^2}\Big) + \log \frac{1}{\delta} + 6\log(N+T)\Big] + \frac{L^2}{8\sqrt{N}}
\end{align*}
Since $\|\theta^{\prime}-\theta\|^2$ is $T$ chi-square distribution, for any positive $t$,
we have
\begin{align*}
    \mathbb{P}\big(\|\theta^{\prime}-\theta\|^2 - T \sigma^2 \geq 2\sigma^2 \sqrt{Tt} + 2t\sigma^2\big) \big) \leq \exp(-t).
\end{align*}
By choosing $t = \frac{1}{2}\log(N)$, with probability $1-N^{-1/2}$, we have
\begin{align*}
    \|\theta^{\prime}-\theta\|^2 \leq \sigma^2 \log(N) + T\sigma^2 + \sigma^2\sqrt{2 T\log(N)} \leq T\sigma^2 \Big(1 + \sqrt{\frac{\log(N)}{T}} \Big)^2.
\end{align*}
 By setting $\sigma = \rho\times \big(\sqrt{T} + \sqrt{\log(N)}\big)^{-1}$, we have $\|\theta^{\prime}-\theta\|^2 \leq \rho^2$. Hence, we get
 \begin{align*}
     \Lc_{\Sc}\Big(\theta^{\prime},\mathcal{N}(\theta,\sigma^2\mathbb{I}_T)\Big) &= \mathbb{E}_{\theta^{\prime}\sim\mathcal{N}(\theta,\sigma^2\mathbb{I}_T)}\mathbb{E}_{\Sc}\big[f_{\theta^{\prime}} \big] = \int_{\|\theta^{\prime}-\theta\|\leq \rho} \mathbb{E}_{\Sc}\big[f_{\theta^{\prime}}\big]d\mathcal{N}(\theta,\sigma^2\mathbb{I}_T)+ \int_{\|\theta^{\prime}-\theta\|> \rho} \mathbb{E}_{\Sc}\big[f_{\theta^{\prime}} \big] d\mathcal{N}(\theta,\sigma^2\mathbb{I}_T)\\
     &\leq \Big(1-\frac{1}{\sqrt{N}} \Big)\max_{\|\theta^{\prime} - \theta\|\leq \rho} \Lc_{\Sc}(\theta^{\prime}) + \frac{1}{\sqrt{N}}L \\
     &\leq \max_{\|\theta^{\prime}-\theta\|_2 \leq \rho} \Lc_{\Sc}(\theta^{\prime}) + \frac{2L}{\sqrt{N}}.
 \end{align*}
Together,
\begin{align*}
    \Lc_{\Dc_{\unm}}^{\unm}(\theta,p_{\mathrm{pos}}) \leq \max_{\theta^{\prime}: \|\theta^{\prime}- \theta\|< \rho} \Lc_{\Sc}^{\unm}(\theta,p_{\mathrm{pos}}) + \frac{1}{\sqrt{N}}\Bigg[ \frac{T}{2}\log\Big(1+\frac{\|\theta\|^2}{T\sigma^2} \Big) + \log\frac{1}{\delta} + O\Big(\log(N+T)\Big) + \frac{L^2}{8} + 2L \Bigg].
\end{align*}

\end{proof}

\begin{thm} \label{thm:sharpness_ideal}
    For $0< \delta < 1$, with the probability at least $1-\delta$ over
the random choice of $\mathcal{S}\sim\mathcal{D}_{\mathrm{un}}^{N}$, we have
the following inequality 

\begin{align*}
&\mathcal{L}_{\mathcal{D}_{\mathrm{un}}}^{\mathrm{un}}\left(\theta,p_{\mathrm{pos}}\right)  \leq\max_{\theta':\Vert\theta'-\theta\Vert<\rho}\mathcal{L}_{\mathcal{S}}^{\mathrm{un}}\left(\theta,p_{\mathrm{pos}}\right)+\quad \frac{1}{\sqrt{N}}\Biggl[  \frac{T}{2}\log\left(1+\frac{\norm{\theta}^{2}}{T\sigma^{2}}\right)+\log\frac{1}{\delta}+\frac{L^{2}}{8}+2L + O\Big(\log(N+T)\Big)\Biggr],
\end{align*}
under the condition that for any $\sigma >0$, $\Lc_{\Dc_{\mathrm{un}}}^{\mathrm{un}}(\theta) \leq \E_{\theta^{\prime}\sim \Nc(\theta,\sigma^2 \mathbb{I}_T)}$, where $L = \frac{2}{\tau} + \log(1+\beta) $ and $\rho = \sigma\big(\sqrt{T} +\sqrt{\log(N)} \big)$.
\end{thm}
\begin{proof}
    The proof is a direct consequence of Theorems \ref{thm:1} and \ref{thm:2}.
\end{proof}

\begin{thm}
\label{thm:practical_sup_un}The following inequality holds

\begin{align} \label{shift_1}
 & \mathcal{L}_{\mathcal{D}_\mathrm{sup}}^\mathrm{sup}\left(\theta\right)\leq\mathcal{L}_{\mathcal{D}_\mathrm{un}}^\mathrm{un}\left(\theta,\tilde{p}_{\mathrm{pos}}\right)-O\left(\frac{1}{\sqrt{K}}\right)
 +\mathcal{L}_\mathrm{shift}\left(\tilde{p}_{\mathrm{pos}},p_{\mathrm{pos}}\right)-\log\beta-O\Big(\frac{1 }{\beta}\Big), 
\end{align} 
where $\mathcal{L}_\mathrm{shift}\left(\tilde{p}_{\mathrm{pos}},p_{\mathrm{pos}}\right)$
is defined as
{\smaller{}
\[
\tau^{-1}\sum_{c=1}^{M}\pi_{c}\mathbb{E}_{x\sim p_{c}}\left[\Vert\mathbb{E}_{x^{+}\sim p_{c}}\left[f_{\theta}\left(x^{+}\right)\right]-\mathbb{E}_{t\sim\mathcal{T},x^{+}=t\left(x\right)}\left[f_{\theta}\left(x^{+}\right)\right]\Vert\right]. 
\]
}
\end{thm}
\begin{proof}
According to Theorem \ref{thm:1}, we have the following
\begin{align} \label{eq:4_1}
\Lc_{\Dc_{\supp}}^{\supp}(\theta)\leq \Lc_{\Dc_{\supp}}^{\supp}(\theta,\bW)\leq \overline{\Lc}_{\Dc_{\text{un}}}^{\text{un}}(\theta)
\end{align}

\begin{align} \label{eq:4_2}
\mathcal{L}_{\mathcal{D}_{\text{un}}}^{\text{un}}(\theta, p_{\mathrm{pos}}) -\overline{\mathcal{L}}_{\mathcal{D}_{\text{un}}}^{\text{un}}(\theta, p_{pos}) =  O\big(\frac{1}{\beta} \big) + O\big(\frac{1}{\sqrt{K}} \big).
\end{align}

\begin{align} \label{eq:4_3}
    \Lc_{\Dc_{\supp}}^{\supp}(\theta) \leq \Lc_{\Dc_{\unm}}^{\unm}(\theta, p_{pos}) - O\Big(\frac{1}{\sqrt{K}} \Big) -\log(\beta)- O\Big(\frac{1}{\beta}\Big).
\end{align}

We now bound the gap $\left|\bar{\mathcal{L}}_{\mathcal{D}_{un}}^{un}\left(\theta,p_{\text{pos}}\right)-\bar{\mathcal{L}}_{\mathcal{D}_{un}}^{un}\left(\theta,p_{\text{pos}}\right)\right|$ as

\begin{align} \label{eq:4_4}
\left|\overline{\mathcal{L}}_{\mathcal{D}_{\mathrm{un}}}^{\mathrm{un}}\left(\theta,p_{\text{pos}}\right)-\overline{\mathcal{L}}_{\mathcal{D}_{\mathrm{un}}}^{\mathrm{un}}\left(\theta,p_{\text{pos}}\right)\right| & =\left|\mathbb{E}_{\left(x,x^{+}\right)\sim p_{\text{pos}}}\left[-\frac{1}{\tau}f_{\theta}\left(x\right).f_{\theta}\left(x^{+}\right)\right]-\mathbb{E}_{\left(x,x^{+}\right)\sim\tilde{p}_{\text{pos}}}\left[-\frac{1}{\tau}f_{\theta}\left(x\right).f_{\theta}\left(x^{+}\right)\right]\right| \nonumber\\
= & \frac{1}{\tau}\left|\mathbb{E}_{\left(x,x^{+}\right)\sim p_{\text{pos}}}\left[f_{\theta}\left(x\right).f_{\theta}\left(x^{+}\right)\right]-\mathbb{E}_{\left(x,x^{+}\right)\sim\tilde{p}_{\text{pos}}}\left[f_{\theta}\left(x\right).f_{\theta}\left(x^{+}\right)\right]\right| \nonumber\\
= & \frac{1}{\tau}\left|\sum_{c=1}^{M}\pi_{c}\mathbb{E}_{x\sim p_{c}}\left[f_{\theta}\left(x\right)\right]\cdot\left\{ \mathbb{E}_{x^{+}\sim p_{c}}\left[f_{\theta}\left(x^{+}\right)\right]-\mathbb{E}_{t\sim\mathcal{T},x^{+}=t(x)}\left[f_{\theta}\left(x^{+}\right)\right]\right\} \right| \nonumber\\
\leq & \frac{1}{\tau}\sum_{c=1}^{M}\pi_{c}\left|\mathbb{E}_{x\sim p_{c}}\left[f_{\theta}\left(x\right)\right]\cdot\left\{ \mathbb{E}_{x^{+}\sim p_{c}}\left[f_{\theta}\left(x^{+}\right)\right]-\mathbb{E}_{t\sim\mathcal{T},x^{+}=t(x)}\left[f_{\theta}\left(x^{+}\right)\right]\right\} \right| \nonumber\\
\leq & \frac{1}{\tau}\sum_{c=1}^{M}\pi_{c}\Vert\mathbb{E}_{x\sim p_{c}}\left[f_{\theta}\left(x\right)\right]\Vert\Vert\mathbb{E}_{x^{+}\sim p_{c}}\left[f_{\theta}\left(x^{+}\right)\right]-\mathbb{E}_{t\sim\mathcal{T},x^{+}=t(x)}\left[f_{\theta}\left(x^{+}\right)\right]\Vert \nonumber\\
\leq & \frac{1}{\tau}\sum_{c=1}^{M}\pi_{c}\mathbb{E}_{x\sim p_{c}}\left[\Vert f_{\theta}\left(x\right)\Vert\right]\Vert\mathbb{E}_{x^{+}\sim p_{c}}\left[f_{\theta}\left(x^{+}\right)\right]-\mathbb{E}_{t\sim\mathcal{T},x^{+}=t(x)}\left[f_{\theta}\left(x^{+}\right)\right]\Vert \nonumber\\
= & \frac{1}{\tau}\sum_{c=1}^{M}\pi_{c}\Vert\mathbb{E}_{x^{+}\sim p_{c}}\left[f_{\theta}\left(x^{+}\right)\right]-\mathbb{E}_{t\sim\mathcal{T},x^{+}=t(x)}\left[f_{\theta}\left(x^{+}\right)\right]\Vert=\mathcal{L}_{\text{shift}}\left(\tilde{p}_{\text{pos}},p_{\text{pos}}\right).
\end{align}

Combining \ref{eq:4_1}, \ref{eq:4_2}, \ref{eq:4_3}, and \ref{eq:4_4}, we reach the conclusion.

\end{proof}

\begin{thm}
\label{thm:practical_sharpness} For $0 < \delta < 1$, with the probability at least $1-\delta$
over the random choices $\mathcal{S}\sim\mathcal{D}_\mathrm{un}^{N}$, we
have the following inequality 
{\small{}
\begin{align*}
 & \mathcal{L}_{\mathcal{D}_\mathrm{sup}}^\mathrm{sup}\left(\theta\right)\leq\max_{\theta':\norm{\theta'-\theta}<\rho}\mathcal{L}_{\mathcal{S}}^\mathrm{un}\left(\theta',\tilde{p}_{\mathrm{pos}}\right)-O\left(\frac{1}{\sqrt{K}}\right)+\frac{1}{\sqrt{N}}\Bigg[\frac{T}{2}\log\left(1+\frac{\norm{\theta}^{2}}{T\sigma^{2}}\right)+\log\frac{1}{\delta}+\frac{L^{2}}{8}+2L + O\Big(\log(N+T) \Big)\Bigg] \\
&\qquad \qquad \qquad +\mathcal{L}_\mathrm{shift}\left(\tilde{p}_{\mathrm{pos}},p_{\mathrm{pos}}\right)-\log\beta-O\Big(\frac{1 }{\beta}\Big)
\end{align*}
} under the condition: for any $\sigma >0$, $\Lc_{\Dc_{\mathrm{un}}}^{\mathrm{un}}(\theta) \leq \E_{\theta^{\prime}\sim \Nc(\theta,\sigma^2 \mathbb{I}_T)}$
where $L=\frac{2}{\tau}+\log(1+\beta)$, $T$ is the number of parameters
in $\theta$, and $\sigma=\frac{\rho}{\sqrt{T}+\sqrt{\log\left(N\right)}}$.
\end{thm}
\begin{proof}
Using the same proof as in Theorem \ref{thm:sharpness_ideal} for $\tilde{p}_{pos}$, we reach

\begin{align*}
&\mathcal{L}_{\mathcal{D}_{\mathrm{un}}}^{\mathrm{un}}\left(\theta,\tilde{p}_{\mathrm{pos}}\right)  \leq\max_{\theta':\Vert\theta'-\theta\Vert<\rho}\mathcal{L}_{\mathcal{S}}^{\mathrm{un}}\left(\theta,\tilde{p}_{\mathrm{pos}}\right)+\quad \frac{1}{\sqrt{N}}\Biggl[  \frac{T}{2}\log\left(1+\frac{\norm{\theta}^{2}}{T\sigma^{2}}\right)+\log\frac{1}{\delta}+\frac{L^{2}}{8}+2L + O\Big(\log(N+T)\Big)\Biggr].
\end{align*}

Further combining with Theorem \ref{thm:practical_sup_un}, we reach the conclusion.

\end{proof}
\section{Additional Experiments on STL-10}

Our previous experiments assume all of our available data is labeled, and thus both the contrastive learning step and the supervised classification training step use the same set of data. However, the realistic scenario would be that we have much less labeled data than unlabeled ones. In that case, we would train our contrastive learning feature extractor on both labeled and unlabeled data, and then only train the final classifier using the labeled one. To evaluate our method in this setting, we opt for using STL-10, a dataset separated into 3 subsets: train, test, and unlabeled. Using the above training method, we run all experiments with the same hyperparameters in the main paper, batch size 256 and for 1000 epochs. Similar to the previous results, DCL and HCL both give slightly worse result comparing to the SimCLR baseline, with HCL having noticeably better robust accuracy than DCL (near 6\%). Meanwhile, our method outperforms all other methods on every metrics; specifically, comparing with standard SimCLR, SSA-CLR improves results by 1\% in clean accuracy, and 5\% in robust accuracy.

\begin{table}[t]
    \centering
    \caption{Test set accuracy from linear evaluations of SSL methods on STL-10 (higher is better).}
    \begin{tabularx}{\linewidth}{*{4}Y}
        \toprule
        Method & {Top-1} & {Top-5} & Robust \\
        \midrule
        SimCLR & 88.85\% & 99.60\% & 54.28\% \\
        Debiased & 86.38\% & 99.23\% & 46.35\% \\
        Hard Neg. & 86.46\% & 99.40\% & 52.28\% \\
        SSA-CLR & \textbf{89.59\%} & \textbf{99.64\%} & \textbf{59.11\%} \\
        \bottomrule
    \end{tabularx}
    \label{tab:imagenette}
\end{table}

\end{document}